\documentclass[final,onefignum,onetabnum]{siamart220329}

\usepackage{lipsum}
\usepackage{amsfonts}
\usepackage{graphicx}
\usepackage{epstopdf}
\usepackage{algorithm}
\usepackage{algpseudocode}
\ifpdf
  \DeclareGraphicsExtensions{.eps,.pdf,.png,.jpg}
\else
  \DeclareGraphicsExtensions{.eps}
\fi

\newsiamremark{remark}{Remark}
\newsiamremark{hypothesis}{Hypothesis}
\crefname{hypothesis}{Hypothesis}{Hypotheses}
\newsiamthm{claim}{Claim}

\headers{Derivative Free Variational Inference}{}

\title{Stable Derivative Free Gaussian Mixture Variational Inference for Bayesian Inverse Problems
\thanks{Submitted to the editors DATE. The authors are in alphabetical order.
\funding{This work was funded by National Natural Science Foundation of China through grant 12471403 and the Fundamental Research Funds for the Central Universities of China.}}}

\author{
Baojun Che\footnotemark[4]~\thanks{School of Mathematical Sciences, Nankai University, Tianjin, China (\email{2111063@mail.nankai.edu.cn}).}
\and
Yifan Chen\thanks{Courant Institute, New York University, NY (\email{yifan.chen@nyu.edu})}
\and 
Zhenghao Huan\thanks{School of Mathematical Sciences, Peking University, Beijing, China (\email{math\_hzh@stu.pku.edu.cn}, \email{wwj66285509350@stu.pku.edu.cn}).}
\and
Daniel Zhengyu Huang\thanks{Corresponding author. Beijing International Center for Mathematical Research,  Center for Machine Learning Research, Peking University, Beijing, China (\email{huangdz@bicmr.pku.edu.cn}).}
\and
Weijie Wang\footnotemark[4]
}

\usepackage{amsopn}

\makeatletter
\newcommand*{\addFileDependency}[1]{
  \typeout{(#1)}
  \@addtofilelist{#1}
  \IfFileExists{#1}{}{\typeout{No file #1.}}
}
\makeatother

\usepackage{appendix}
\usepackage{subfig}
\usepackage{graphicx}
\usepackage{color}
\usepackage{bm}
\usepackage{enumitem}
\newcommand{\N}{\mathcal{N}}

\newcommand{\G}{\mathcal{G}}
\newcommand{\F}{\mathcal{F}}
\newcommand{\E}{\mathbb{E}}

\newcommand{\I}{I}
\newcommand{\R}{\mathbb{R}}
\newcommand{\T}{\mathcal{T}}

\newcommand{\rhoa}{a}

\newcommand{\fM}{\mathfrak{M}}

\DeclareMathOperator*{\argmin}{arg\,min}

\newcommand{\dd}{\mathrm{d}}

\newenvironment{newremark}[1]{%
    \begin{remark}#1}{%
    \Endofdef\end{remark}%
}
\newcommand{\xqed}[1]{%
    \leavevmode\unskip\penalty9999 \hbox{}\nobreak\hfill
    \quad\hbox{\ensuremath{#1}}}
\newcommand{\Endofdef}{\xqed{\lozenge}}

\definecolor{darkred}{rgb}{.6,0,0}
\definecolor{darkblue}{rgb}{0,0,.7}
\definecolor{darkgreen}{rgb}{0,.7,0}
\definecolor{darkbrown}{rgb}{0.8,0.4,0.4}

\begin{document}

\maketitle
\begin{abstract}
This paper is concerned with the approximation of probability distributions known up to normalization constants, with a focus on Bayesian inference for large-scale inverse problems in scientific computing. In this context, key challenges include costly repeated evaluations of forward models, multimodality, and inaccessible gradients for the forward model. To address them, we develop a variational inference framework that combines Fisher-Rao natural gradient with specialized quadrature rules to enable derivative free updates of Gaussian mixture variational families. The resulting method, termed Derivative Free Gaussian Mixture Variational Inference (DF-GMVI), guarantees covariance positivity and affine invariance, offering a stable and efficient framework for approximating complex posterior distributions. The effectiveness of DF-GMVI is demonstrated through numerical experiments on challenging scenarios, including distributions with multiple modes, infinitely many modes, and curved modes in spaces with up to 100 dimensions. The method's practicality is further demonstrated in a large-scale application, where it successfully recovers the initial conditions of the Navier-Stokes equations from solution data at positive times.
\end{abstract}



\section{Introduction}
Sampling a target probability distribution known up to a normalization constant is a classical problem in scientific computing.
Specifically, in Bayesian inverse problems \cite{kaipio2006statistical,stuart2010inverse}, the goal is to recover an unknown parameter 
$\theta \in \R^{N_{\theta}}$ from noisy observation $y \in \R^{N_y}$, 
governed by the equation
\begin{equation}
\label{eq:IPs}
    y = \G(\theta) + \eta.
\end{equation}
Here, $\G$ represents a forward map which, for the problems we consider, is nonlinear and requires solving a partial differential equation (PDE) for each evaluation. The observation noise $\eta$ follows a Gaussian distribution: $\eta \sim \N(0,\Sigma_{\eta})$. Within the Bayesian framework, we assign a Gaussian prior $\N(r_0, \Sigma_0)$ to the unknown parameter $\theta$, resulting in a posterior distribution from which we aim to draw samples
\begin{align}
\label{eq:PhiR}
    \rho_{\rm post}(\theta) & \propto  \exp(-\Phi_R(\theta)),\ 
\Phi_R(\theta)  = \frac{1}{2}\lVert\Sigma_{\eta}^{-\frac{1}{2}}(y - \G(\theta)) \rVert^2+\frac{1}{2}\lVert\Sigma_{0}^{-\frac{1}{2}}(\theta - r_0) \rVert^2.
\end{align}
We can also express $\Phi_R$ through a nonlinear least-squares form, with an augmented map $\F(\theta)$ satisfying
\begin{align}
\Phi_R(\theta) = \frac{1}{2} \F(\theta)^T\F(\theta),\ \ 
\label{eq:Ftheta}
   \F(\theta) =
\begin{bmatrix}
    \Sigma_\eta^{-\frac{1}{2}}\bigl(y- \G(\theta)\bigr)\\
    \Sigma_0^{-\frac{1}{2}}\bigl(r_0 - \theta\bigr)
\end{bmatrix}.
\end{align}

\subsection{Challenges} For many Bayesian inverse problems in scientific applications, computing gradients of $\Phi_R$ might be infeasible, as it requires derivatives of large-scale PDE-based models $\G$ that may be black-box (e.g., climate models~\cite{sen2013global,yatunin2025climate,guo2024ib}), use non-differentiable numerical methods (e.g., embedded boundary methods~\cite{peskin1977numerical,huang2018family,huang2020modeling,cao2022bayesian} and adaptive mesh refinement~\cite{berger1989local,borker2019mesh}), or model discontinuous physics (e.g., in fracture~\cite{moes1999finite} or cloud modeling~\cite{tan2018extended,lopez2022training}). Derivative free methods are thus needed; see a review of methodologies in \cref{sec-review-derivative-free}. 

While there are a few Markov Chain Monte Carlo (MCMC) and Sequential Monte Carlo (SMC) approaches that do not require gradients, they often require numerous function evaluations, particularly in high dimensions, to achieve convergence or mitigate weight collapse. This makes them impractical in many large scale inverse problems given the high computational cost of evaluating $\mathcal{G}$. The computational cost can be reduced by (online) constructing surrogate models \cite{gao2024adaptive,li2023surrogate,yan2019adaptive,cleary2020calibrate,yan2021stein} and/or employing multifidelity strategies~\cite{giles2015multilevel,nagel2016unified,alsup2023context} for the forward model $\mathcal{G}$, but additional sources of error might be introduced.
Furthermore, the potential multimodality of $\rho_{\rm post}(\theta)$ can cause MCMC methods to struggle with mode transitions~\cite{gayrard2004metastability,gayrard2005metastability} and complicates surrogate model training. Missing modes can lead to significant prediction errors in scientific applications~\cite{tebaldi2005quantifying}.

Variational inference offers a promising scalable alternative, with black-box variational inference (BBVI)~\cite{ranganath2014black} being a popular derivative free approach. However, BBVI typically relies on stochastic Monte Carlo approximations to estimate gradients, which exhibit high variance. This high variance can make BBVI unreliable, often necessitating variance reduction techniques and extensive time step tuning. Even with these improvements, small time steps remain necessary to maintain stability; see a demonstration in \cref{ssec:Multi-Dimensional-Problems}.

On the other hand, Kalman methodology has also been used to develop derivative free variational inference methods \cite{chen2024efficient,huang2022efficient} based on approximations inspired by ensemble Kalman filters. While these approaches can be effective for posterior distributions close to Gaussian or Gaussian mixtures with well-separated components, they often become inaccurate or even unstable in the presence of non-Gaussian or strongly curved modes; see the experiments reported in \cite[Appendix E]{chen2024efficient} and detailed discussion in \cref{sec-Numerical challenges for solving the gradient flow}.

Motivated by these challenges, in this work, we explore guidelines for designing stable derivative free Gaussian mixture variational inference methods. We take the perspective that both BBVI \cite{ranganath2014black} and the Kalman methodology-based approaches \cite{chen2024efficient,huang2022efficient} are numerically approximating a gradient flow equation. In particular, with natural gradients, the flow equation becomes the Fisher-Rao gradient flow. Our guidelines address stable numerical approximation of the Fisher-Rao gradient flow. Building on the guidelines, we propose a novel approach that requires no hyperparameter tuning while achieving both high accuracy and computational efficiency.

\subsection{Contributions}
Specifically, our contributions are as follows:
\begin{enumerate}

    \item  

   We introduce practical guidelines for numerically solving the Fisher-Rao gradient flow within the Gaussian mixture variational family. These guidelines emphasize the consistent integration of entropy and cross-entropy terms to ensure accuracy, and the maintenance of covariance matrix positivity as a critical indicator of numerical stability. They provide a foundation for enhancing the robustness of existing Gaussian mixture variational inference methods and developing new algorithms.

    \item 

    Building on these guidelines, we propose a derivative free variational inference approach (DF-GMVI), which employs novel quadrature rules to approximate expectations involving Gaussian mixtures and $\Phi_R$. These quadrature rules go beyond linear approximations by capturing curvature information of the augmented map $\F$, while maintaining linear computational cost with respect to the dimensionality of $\theta$ and requiring no derivatives. DF-GMVI ensures \textit{covariance positivity} and exhibits \textit{affine invariance}, both of which contribute to its superior stability, even with large time steps.

    \item We demonstrate that DF-GMVI effectively captures multiple or even infinite modes, as well as curved modes, in model problems involving up to 100 dimensions. It also performs well in PDE-based applications, such as reconstructing the Navier-Stokes initial condition from solution data at positive times.
\end{enumerate}

\subsection{Literature Review} 
\label{ssec:relatedwork}


\subsubsection{Variational Inference} 

The main idea of variational inference is to approximate a target density $\rho_{\rm post}$ within a variational family of densities $Q$ from the view of optimization, that is, to identify the member of this family that minimizes an energy function $\mathcal{E}$:
$$\rho=\argmin_{\rho \in Q}\mathcal{E}(\rho,\rho_{\rm post}),$$
  where the minimizer coincides with $\rho_{\rm post}$.
The energy function is often chosen as the Kullback-Leibler divergence \eqref{eq:KL}. This optimization problem is then addressed via gradient flows and their discretizations. 
Along the gradient flow of the energy function, $\frac{\dd \rho}{\dd t} = - \nabla_{M} \mathcal{E}(\rho,\rho_{\rm post})$, with respect to a selected metric $M$, the energy gradually decreases.

In practice, we can classify variational inference methods into two categories: non-parametric and parametric approaches. In non-parametric variational inference, 
the variational distribution is often represented by a particle system. Gradient flows with various metrics \cite{chen2023sampling} have been used in this context, including the Wasserstein gradient flow \cite{jordan1998variational,carrillo2019blob,lambert2022variational}, Fisher-Rao gradient flow \cite{maurais2024sampling, DomingoEnrich2023AnEE,carrillo2024fisher,zhu2024kernel}, Wasserstein-Fisher-Rao gradient flow \cite{lu2023birth}, Kalman-Wasserstein gradient flow \cite{garbuno2020interacting} and Stein gradient flow \cite{liu2017stein,yan2021stein,jia2022stein}.
The particle system evolves according to the gradient flows and gradually approximates the target distribution. The convergence rate and accuracy depend on the number of particles used to represent the distribution and the quality of their representation.

In parametric variational inference,  the variational distribution is a parametric density. A common choice is the Gaussian distribution, which leads to Gaussian variational inference~\cite{opper2009variational,lambert2022variational,blei2017variational}.
To account for the geometric structure of parameter space, the gradient is preconditioned by the Fisher information matrix, resulting in the Fisher-Rao \textit{natural gradient}~\cite{amari1998natural,opper2009variational} methods. This approach has been shown to outperform standard gradient descent \cite{james2020new}. Researchers have also explored Gaussian approximations for Stein gradient flows \cite{liu2024towards} and Wasserstein gradient flows \cite{lambert2022variational,diao2023forward}.
However, the Gaussian approximation offers limited expressive power, which has motivated the use of more flexible variational families, such as Gaussian mixtures. These include Gaussian mixture approximations of the natural gradient flow \cite{lin2019fast,chen2024efficient} and Wasserstein gradient flow \cite{lambert2022variational,huix2024theoretical}.
Specifically, the present work focuses on Gaussian mixture approximation of the natural gradient flow, emphasizing stable and derivative free approximation methods that enhance computational efficiency.

\subsubsection{Derivative Free Sampling Approaches}

\label{sec-review-derivative-free}

A large class of derivative free sampling approaches is based on Markov chain Monte Carlo (MCMC) \cite{geyer1992practical, gelman1997weak}, where a derivative free proposal is used to move particles. The main challenges of these methods are the absence of stopping criteria and slow convergence, which worsen as the dimensionality of the problem increases. Several improvements have been proposed to address these issues, such as enhancing the proposal distribution by preserving affine invariance (e.g., the stretch move method \cite{goodman2010ensemble}) and leveraging parallelization (e.g., with multiple chains \cite{foreman2013emcee,braak2006markov, vrugt2009accelerating}). This work also incorporates these two strategies.

Another class of methods is Sequential Monte Carlo (SMC) \cite{doucet2009tutorial, smith2013sequential, beskos2015sequential,lu2025sequential}, where particles and their associated weights are updated together using importance sampling. The importance weights, especially in high dimensions, can suffer from collapse issues , and techniques such as frequent resampling are needed to address particle degeneracy, sample impoverishment, and instability \cite{elfring2021particle}.

Finally, in variational inference, a major derivative free approach is black-box variational inference \cite{ranganath2014black}, which employs stochastic approximation via Monte Carlo methods to estimate gradients. However, the high variance of these gradient estimates often results in unstable updates and slow convergence. To address these challenges, researchers have proposed strategies such as incorporating variance reduction techniques and adaptive learning rates to stabilize gradient estimates \cite{ranganath2014black, welandawe2024framework}.
Alternatively, gradient-based variational inference methods, such as Kalman-Wasserstein gradient flow \cite{garbuno2020interacting} and Stein gradient flow \cite{han2018stein}, avoid direct gradient computation by leveraging Stein's lemma: \[\E_{\N(\theta; m, C)} [\nabla_{\theta} \F(\theta)] = C^{-1} {\rm Cov}_{\N(\theta; m, C)}[\F(\theta), \theta].\] This approach enables gradient estimation through quadrature rules derived from Kalman filtering techniques, such as the ensemble Kalman filter~\cite{evensen1994sequential}, the unscented Kalman filter \cite{julier1997new,julier2000new,huang2022iterated}, and the cubature Kalman filter \cite{arasaratnam2009cubature}. 
Notably, the latter two achieve exactness for linear $\F$.
Despite these advancements, these approaches still require small time steps to mitigate instability, especially when the posterior is high-dimensional, exhibits multimodality, or involves complex dependencies.
The present work advances quadrature rules for estimating gradients and Hessians, which are crucial for developing stable, derivative free variational inference methods.

\subsection{Organization}
\label{ssec:over}
In \cref{sec:ngvi}, we provide an overview of natural gradient variational inference, with a focus on Gaussian mixture variational families. 
\Cref{sec:guidelines} outlines practical guidelines for numerically solving natural gradient flow.
\Cref{sec:df-gmvi} introduces our Derivative Free Gaussian Mixture Variational Inference (DF-GMVI), and the related theoretical insights are presented in \cref{sec:theory}. Numerical experiments are described in \cref{sec:numerics}, which serve to empirically validate the theory and demonstrate the effectiveness of the proposed framework for Bayesian inference. Finally, concluding remarks are provided in \cref{sec:conclusion}.

\section{Variational Inference with Natural Gradient}
\label{sec:ngvi}
 In this section, we briefly review variational inference with natural gradient methods, from the perspective of gradient flows. We focus on Gaussian mixture variational families. 
 
 In detail, we approximate the posterior distribution \cref{eq:PhiR} by minimizing the Kullback–Leibler (KL) divergence~\cite{wainwright2008graphical,blei2017variational}
\begin{equation}
\label{eq:KL}
    {\rm KL}[\rho_a \Vert \rho_{\rm post}] = \int \rho_a \log\Bigl(\frac{\rho_a}{\rho_{\rm post}}\Bigr)\dd\theta
\end{equation}
over a family of variational  densities $\rho_a$, parameterized by $a\in \R^{N_a}$. In Gaussian mixture variational inference, the variational family is a $K$-component Gaussian mixture \[\rho_a^{\rm GM}(\theta) = \sum_{k=1}^{K} w_k \N(\theta; m_k, C_k),\] parameterized by means $m_k\in\R^{N_\theta}$, covariances $C_k\in\R^{N_\theta \times N_\theta}$ and weights $w_k \in \mathbb{R}_{\geq 0}$, collectively denoted by the parameter vector \[a:=[m_1, ..., m_k, ..., m_K, C_1, ..., C_k, ..., C_K, w_1, ..., w_k, ..., w_K].\] 
Weights satisfy $\sum_{k=1}^{K} w_k = 1$.

The minimization of \eqref{eq:KL} can be done by using gradient flows. The natural gradient method~\cite{amari1998natural} employs the Fisher information matrix~\cite{rao1945information} as a preconditioner:
\begin{equation*}
    \fM(\rhoa) = {\rm FIM}(a):= \int  \frac{\partial \log \rho_{\rhoa}(\theta)}{\partial \rhoa} \frac{\partial \log \rho_{\rhoa}(\theta)}{\partial \rhoa}^T \rho_\rhoa(\theta) \dd \theta.
\end{equation*}
This leads to the following gradient flow for updating $a$:
\begin{align}\label{eq:GGVI}
    \frac{\dd a}{\dd t} = -\fM(\rhoa)^{-1}\nabla_a  {\rm KL}[\rho_a \Vert \rho_{\rm post}].
\end{align}
This equation can also be understood as a Gaussian mixture approximation of the Fisher-Rao gradient flow in the space of probability densities \cite{chen2023sampling}.

\subsection{Approximate Fisher Information} For the Gaussian mixture variational family, the Fisher information matrix lacks a close-form  expression, and its inversion is computationally challenging. To address this, following~\cite[Appendix C.8]{chen2024efficient}\cite{lin2019fast}, we approximate the Fisher information matrix with a block-diagonal form. Substituting this approximation into \cref{eq:GGVI} leads to the following equations:
\begin{equation}
\begin{split}
    \label{eq:Appr-FR-GM}
        \frac{\dd m_{k}}{\dd t} 
        &= -C_k\int \N_k(\theta) \Bigl( \nabla_{\theta} \log\rho_a^{\rm GM}  +  \nabla_{\theta} \Phi_R \Bigr)  \dd\theta,
        \\
        \frac{\dd  C_{k}^{-1}}{\dd t} 
        &= \int \N_k(\theta)\bigl(\nabla_{\theta}\nabla_{\theta}\log \rho_a^{\rm GM}  + \nabla_{\theta}\nabla_{\theta}\Phi_R\bigr) \dd\theta,
        \\
        \frac{\dd \log w_{k}}{\dd t} &= -\int \Bigl(\N_k(\theta) -  \rho_a^{\rm GM}\Bigr)\bigl(\log \rho_a^{\rm GM}  + \Phi_R \bigr) \dd\theta. 
\end{split}
\end{equation}
A detailed derivation of \cref{eq:Appr-FR-GM}, along with its connection to natural gradient descent for Gaussian variational families, is presented in the supplementary material.
 In the following, we refer to \cref{eq:Appr-FR-GM}  as the natural gradient flow, although it involves an approximation of the Fisher information matrix.  
 
 Notably, \cref{eq:Appr-FR-GM} is \textit{affine invariant}---its behavior remains unchanged under any invertible affine transformation, or equivalently, in coordinate systems related by such transformations (see \cref{prop:affine}). 
This invariance is particularly beneficial for improving the efficiency of sampling methods when handling highly anisotropic posterior distributions~\cite{goodman2010ensemble,foreman2013emcee,chen2023sampling}.

\subsection{Numerical Challenges for Solving the Gradient Flow}
\label{sec-Numerical challenges for solving the gradient flow}Solving \cref{eq:Appr-FR-GM} involves evaluating several Gaussian integrals on the right hand side. 
In the pioneering work~\cite{lin2019fast}, Monte Carlo integration was employed to compute these Gaussian integrals, which require both the gradient and Hessian of $\Phi_R$. However, to ensure numerical stability, very small time steps were necessary.
A subsequent study~\cite{chen2024efficient} proposed an efficient, derivative free approximation. In this approach, Monte Carlo integration was applied only to the Gaussian integrals associated with the logarithm of the Gaussian mixture, $\log\rho_a^{\rm GM}$, as these functions do not involve the evaluation of forward models, are relatively inexpensive to compute, and admit closed-form expressions. For the integrals involving $\Phi_R$, a Kalman-based approach, specifically the unscented transformation \cite{julier2000new,huang2022iterated}, was applied. 
This strategy preserved covariance positivity, enabling significantly larger time steps and resulting in an efficient, derivative free posterior approximation method known as Gaussian Mixture Kalman Inversion (GMKI). 

Although GMKI effectively captures multiple modes in the posterior, it struggles to accurately represent the density in regions where different modes significantly overlap (see \cite[Section 6.1, Appendix E]{chen2024efficient}). 
We identify this limitation as resulting from the inconsistent approximation of the two terms: $\log\rho_a^{\rm GM}$ and $\Phi_R$, which we refer to as the entropy and cross-entropy terms because of their origin in \eqref{eq:KL}.
Since the stationary point of \cref{eq:Appr-FR-GM} requires a balance between these terms, consistent treatment is important. This issue may be connected in spirit to the phenomenon pointed out by \cite{wibisono2018sampling}, where inconsistency in temporal discretization between the entropy and cross-entropy terms was argued to introduce asymptotic bias when discretizing Langevin dynamics in time.

In the present work, we focus on consistent spatial approximation of \cref{eq:Appr-FR-GM} through quadrature.
To do so, in the following section, we first consider some numerical examples to investigate the numerical issues in detail. These investigations motivate our practical guidelines for designing accurate and efficient solution schemes.

\section{Numerical Motivations and Practical Guidelines}
\label{sec:guidelines}
We discuss practical guidelines for the numerical solution of the natural gradient flow~\cref{eq:Appr-FR-GM}, specifically about the parametric approximation and Gaussian integrals. To this end, we consider illustrative 2D examples, which include a 3-mode Gaussian mixture target density, with 
$\Phi_R(\theta) = - \log\Bigl( \sum_{i=1}^3 w^{*}_i \mathcal{N}(\theta; m^{*}_i, C^{*}_i) \Bigr),$
where $w^{*}_1 = 0.3,\ m^{*}_1 = \begin{bmatrix}1 \\2\end{bmatrix},\ C^{*}_1 = I/4$,  
$w^{*}_2 = 0.4,\ m^{*}_2 = \begin{bmatrix}2 \\1\end{bmatrix},\ C^{*}_2= I/4$, and
$w^{*}_3 = 0.3,\ m^{*}_3 =\begin{bmatrix}-1 \\-1\end{bmatrix},\ C^{*}_3 = I/4$,
and a circular-shaped target density given by 
$\Phi_R(\theta) = \frac{(1 - \lVert \theta \rVert^2)^2}{2\sigma_\eta^2}$ with $\sigma_\eta=0.3$,
as described in \cite[Appendix E]{chen2024efficient}.

We use Gaussian mixtures with $K=10$ components as the parametric family and solve the natural gradient flow \cref{eq:Appr-FR-GM} using the forward Euler scheme with a fixed time step size. 
For this investigation, we will assume access to both the gradient and Hessian of $\log\rho_a$ and $\Phi_R$.
The Gaussian integrals in~\cref{eq:Appr-FR-GM} are evaluated using various quadrature rules, including the \textit{mean-point approximation}: 
\begin{equation}
\label{eq:mean_point}
    \int \N(\theta; m,C) g(\theta)\dd\theta \approx \int \N(\theta; m,C) \Bigl( g(m) + \nabla g(m) (\theta - m) \Bigr)\dd\theta =  g(m),  
\end{equation}
where $g$ may be a scalar-, vector-, or matrix-valued function. This corresponds to the linear approximation used in the 
extend Kalman filter~\cite{jazwinski2007stochastic}.
Other approaches include the unscented transformation~\cite{julier1995new,julier2000new}, which employs a specialized quadrature rule with $2N_\theta + 1$ sigma points, and the Monte Carlo method~\cite{hoffman2013stochastic} with $J=20$ ensembles. 

We consider four combinations for the choices of quadratures: (1) all $\Phi_R$ and $\log\rho_a$ related terms are integrated with mean-point approximation,
(2) all terms are integrated with unscented transformation,
(3) all terms are integrated with Monte Carlo method,
and (4) $\Phi_R$ related terms are integrated with mean-point approximation, while $\log\rho_a$ related terms are integrated with the Monte Carlo method. We note that the fourth combination involves an inconsistent treatment of the integrals.

\begin{figure}[ht]
\centering
    \includegraphics[width=0.98\textwidth]{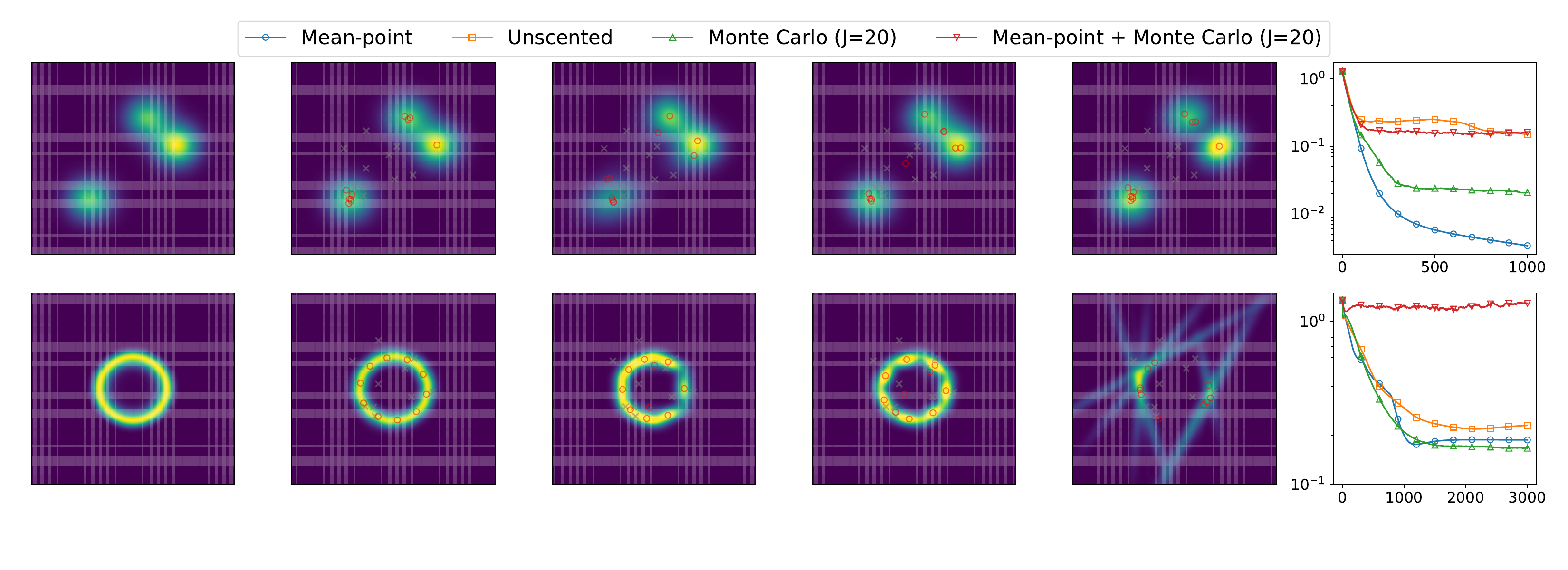}
    \caption{Quadrature rule comparison for Gaussian mixture and circular-shaped target densities.
    Each panel displays the reference density and the estimates from approximating the natural gradient flow~\cref{eq:Appr-FR-GM} with different quadrature rule combinations (left to right). Red circles represent the means $m_k$
  of each Gaussian component, with gray crosses marking their initial positions. The final panel shows the total variation distance between the reference density and the estimated densities at each iteration.}
    \label{fig:Quadrature-Rules}
\end{figure}

The results are presented in \cref{fig:Quadrature-Rules}. We highlight three observations. First, the fifth panels in both tests demonstrate the drawback of the inconsistent approximation of the integrals with respect to $\log\rho_a^{\rm GM}$ and $\Phi_R$. This combination fails to accurately approximate the delicate regions where different modes significantly overlap. In contrast, a consistent treatment leads to better performance. In fact, such consistent treatment ensures that replacing $\log\rho_a^{\rm GM}$ with the exact solution makes the right hand side of \cref{eq:Appr-FR-GM} (after the quadrature is applied) equal to zero exactly.

Second, we note that the maximum time steps for stable simulation in these two examples are $3\times10^{-2}$ and $6\times10^{-3}$, respectively. For larger time steps, the covariance matrix update loses positivity. For more challenging problems (e.g. Case D in \cref{ssec:Multi-Dimensional-Problems}), the maximum stable time step can drop to $10^{-7}$. The positivity of the covariance thus serves as an indicator of stable simulation.

Finally,  we observe that the mean-point approximation \eqref{eq:mean_point} is the most efficient, as it requires only one function evaluation to compute the integral and performs surprisingly well compared to other quadrature rules. 
When this approximation is applied, at the stationary point of the method, we have the following conditions:
\begin{equation}
\label{eq:mean_point_stationary}
\begin{aligned}
      \nabla_\theta\log\rho_a^{\rm GM}(m_k) + \nabla_\theta\Phi_R(m_k)  = 0,
      \\
    \nabla_\theta\nabla_\theta\log\rho_a^{\rm GM}(m_k) + \nabla_\theta\nabla_\theta\Phi_R(m_k)  = 0 ,
    \\ 
    \log\rho_a^{\rm GM}(m_k) +  \Phi_R(m_k)  = \text{const} . 
\end{aligned}
\end{equation}
These conditions are not the same as the exact one (i.e., $\log\rho_a^{\rm GM}(\theta) + \Phi_R(\theta) = \text{const}$) due to approximation error. However, these conditions at least guarantee that the exact condition is satisfied at these mean points $\{m_k\}$. In contrast, when other quadrature rules are used, the stationary conditions become linear combinations of $\rho_a^{\rm GM}$ and $\Phi_R$ that satisfy \cref{eq:mean_point_stationary} (e.g., the Monte Carlo quadrature rule imposes $\frac{1}{J}\sum_{j}\nabla_\theta\log\rho_a^{\rm GM}(\theta^{(j)}_k) + \frac{1}{J}\sum_{j}\nabla_\theta\Phi_R(\theta^{(j)}_k) = 0$, where $\theta^{(j)}_k\sim \mathcal{N}(m_k,C_k)$). This can lead to larger variance and might introduce cancellation errors, making it less indicative that the exact condition is satisfied at any point. The third and fourth panels in the second row of \cref{fig:Quadrature-Rules} show that the unscented and Monte Carlo approximations perform worse than the mean-point approximation.

These observations suggest the practical guidance of developing quadrature rules that ensure consistency and guarantee covariance positivity based on the mean-point approximation \eqref{eq:mean_point}.

\section{Derivative Free Gaussian Mixture Variational Inference}
\label{sec:df-gmvi}

In this section, we introduce a stable, derivative free approximation of~\cref{eq:Appr-FR-GM} for Bayesian inverse problems in~\cref{eq:PhiR,eq:Ftheta}, termed Derivative Free Gaussian Mixture Variational Inference (DF-GMVI). 
The DF-GMVI method employs two \textit{derivative free} quadrature rules, defined in \cref{def:DF-quadrature,def:GM-quadrature}, motivated by the guidelines outlined in \cref{sec:guidelines}. Specifically,
building on the consistent mean-point approximation \cref{eq:mean_point}, we incorporate additional corrections to approximate Hessian expectations, capturing as much curvature information of $\F$ as possible, while maintaining linear complexity in the evaluation of $\F$. These corrections ensure \textit{covariance positivity} and maintain \textit{affine invariance}, which contribute to the superior stability of the overall algorithm, as detailed in \cref{sec:theory}.  

We begin by deriving a specialized quadrature rule to compute expectations of the following forms:
\begin{align}
\label{eq:EPhi_R_terms}
   \E_{\N} \bigl[\Phi_R\bigr] , \, \E_{\N} \bigl[\nabla_\theta \Phi_R\bigr],\, \textrm{and} \,\E_{\N} \bigl[\nabla_\theta \nabla_\theta \Phi_R\bigr]
\end{align}
with respect to the Gaussian density $\N(\theta; m, C)$ in a derivative free manner. Here, we assume access only to $\F(\theta)$, and recall that $\Phi_R(\theta) = \frac{1}{2}\F(\theta)^T\F(\theta)$ has a nonlinear least-squares structure.
\begin{definition}
\label{def:DF-quadrature}
    Given $\theta\sim \N(\theta; m, C) \in \R^{N_{\theta}}$ and a hyperparameter $\alpha > 0$, we generate $2N_\theta+1$ quadrature points
    \begin{align*}
        \theta_0 = m ,\qquad \theta_i = m + \alpha [\sqrt{C}]_i ,\qquad \theta_{N_\theta + i} = m - \alpha [\sqrt{C}]_i \quad (1 \leq i \leq N_\theta),
    \end{align*}
    where $\sqrt{C}$ is the square root matrix of $C$, such that 
    $C = \sqrt{C} \sqrt{C}^T$ and $[\sqrt{C}]_i$ denotes its $i$-th column.  Given $\F : \R^{N_\theta} \rightarrow \R^{N_x}$, we compute vectors
    \begin{align}
    \label{eq:qr_vectors}
        c = \F(\theta_0), \
        b_i = \frac{\F(\theta_i) - \F(\theta_{N_\theta + i})}{2\alpha}, \
        a_i = \frac{\F(\theta_i) + \F(\theta_{N_\theta + i}) - 2\F(\theta_0)}{2\alpha^2} \ (1 \leq i \leq N_\theta),
    \end{align}
    and denote $B = [b_1;\,b_2;\,\cdots b_{N_\theta}] \in \R^{N_x \times N_\theta}$ and $A = [a_1;\,a_2;\,\cdots a_{N_\theta}] \in \R^{N_x \times N_\theta}$.
The expectation of the function is then approximated as 
 \begin{align}
 \label{eq:Function}
    \E_{\N}[\Phi_R ] = \frac{1}{2}\E_{\N}\Bigl[ \F(\theta)^T \F(\theta) \Bigr] \approx \frac{1}{2}  c^Tc .
\end{align}
The expectation of the gradient is approximated as 
 \begin{align}
 \label{eq:Gradient}
    \E_{\N}[\nabla_\theta   \Phi_R ] = \frac{1}{2}\E_{\N}\Bigl[\nabla_\theta \bigl(\F(\theta)^T \F(\theta)\bigr)\Bigr] \approx
      \sqrt{C}^{-T}  B^Tc.
\end{align} 
    The expectation of the Hessian is approximated as
    \begin{align}
    \label{eq:Hessian}
    \E_{\N}[\nabla_\theta\nabla_\theta  \Phi_R ] = \frac{1}{2}\E_{\N}\Bigl[\nabla_\theta\nabla_\theta \bigl(\F(\theta)^T \F(\theta)\bigl)\Bigr] \approx
    \sqrt{C}^{-T}(6\mathrm{Diag}(A^TA) + B^T B)\sqrt{C}^{-1},
\end{align}
where $\mathrm{Diag}(\cdot)$ extracts the diagonal elements of a matrix to form a diagonal matrix.
\end{definition}

To derive \cref{def:DF-quadrature}, we first define $\widetilde{\F}(\widetilde\theta) = \F(m + \sqrt{C}\widetilde{\theta})$.
The expectations of $\Phi_R(\theta)$ in \cref{eq:Function} and its gradient in \cref{eq:Gradient} are then approximated using the mean-point approximation with the finite difference approximation of the gradient $\nabla_{\widetilde\theta}   \widetilde\F(0)^T$, as given in~\cref{eq:qr_vectors}. 
The expectation of the Hessian in \cref{eq:Hessian} can be interpreted as a correction that incorporates curvature information into the mean-point approximation in the Gaussian Newton formulation, $\sqrt{C}^{-T}( B^T B)\sqrt{C}^{-1}$.
The derivation of the additional terms $\sqrt{C}^{-T}(6\mathrm{Diag}(A^TA))\sqrt{C}^{-1}$ proceeds as follows:
We assume
that $\widetilde\F(\widetilde \theta) = A \widetilde{\theta} \odot \widetilde{\theta} + B \widetilde{\theta} + c $, where $\theta\odot \theta$ denotes the Hadamard product.  The column vectors of matrices $A$ and $B$ are approximated by finite differences given in \cref{eq:qr_vectors}.
Under this assumption, the expectation of the Hessian has an analytical form:
\begin{equation*}
    \begin{split}
    \frac{1}{2}\E_{\N}\Bigl[\nabla_\theta\nabla_\theta \bigl(\F(\theta)^T \F(\theta)\bigl)\Bigr] 
    &=  \sqrt{C}^{-T} \E_{\widetilde{\N}(0,I)}[\nabla_{\widetilde\theta}\nabla_{\widetilde\theta}  \widetilde\F(\widetilde\theta)^T\widetilde\F(\widetilde\theta) + \nabla_{\widetilde\theta}  \widetilde\F(\widetilde\theta)^T \nabla_{\widetilde\theta}\widetilde\F(\widetilde\theta)] \sqrt{C}^{-1}   \\
    &= 
    \sqrt{C}^{-T}(D + 2\mathrm{Diag}(A^T c) + B^T B)\sqrt{C}^{-1},
    \end{split}
\end{equation*}
where $D$ is a diagonal matrix with 
$D_{jj} = 2\sum_{k}a_{j}^Ta_{k} + 4 a_{j}^Ta_{j}$. To ensure covariance positivity (see \cref{prop:positivity}),  we retain only the positive definite part of this expectation, resulting in  \cref{eq:Hessian}.

\begin{newremark}
As a comparison, BBVI~\cite{ranganath2014black} employs integration by parts to express the terms as follows:
\begin{subequations}
\begin{align} 
&\E_{\N}[\nabla_\theta \Phi_R(\theta)]    = C^{-1}\E_{\N}\bigl[(  \theta - m )\Phi_R\bigr], \label{eq:gradient-BBVI}\\
&\E_{\N}[\nabla_\theta\nabla_\theta \Phi_R(\theta)]    = C^{-1}\E_{\N}\Bigl[\bigl( (\theta - m)(\theta - m)^T - C\bigr)\Phi_R\Bigr] C^{-1}.\label{eq:hessian-BBVI}
\end{align}
\end{subequations}
BBVI applies Monte Carlo approximations of the above two expectations, which often lead to high variance.

On the other hand, applying the unscented transformation~\cite{julier1995new,julier2000new} for evaluating \cref{eq:gradient-BBVI} also gives a finite difference approximation of $\nabla_\theta\Phi_R(m)$ as in \Cref{def:DF-quadrature}. However, \cref{eq:hessian-BBVI} involves high-order matrix-valued functions, and directly applying unscented~\cite{julier1995new,julier2000new}, cubature~\cite{arasaratnam2009cubature} transformations, or Monte Carlo result in high-variance or non-positive estimates which are less accurate or stable numerically.  
\end{newremark}

Then, we discuss the quadrature rules for the term involving $\log \rho_a^{\rm GM}$. Note that the gradient and Hessian of $\log\rho_a$ are accessible and inexpensive to compute here. We develop quadrature rules to approximate
the expectation of $\log \rho_a^{\rm GM}$ and its derivatives used in \cref{eq:Appr-FR-GM}, following the guidelines discussed in \cref{sec:guidelines}, with particular emphasis on ensuring the consistency with \cref{def:DF-quadrature}.


\begin{definition}
\label{def:GM-quadrature}
Given Gaussian mixture 
$\rho_a^{\rm GM}(\theta) = \sum_{k=1}^{K} w_k \N(\theta; m_k, C_k),$
the expectation of $\log \rho_a^{\rm GM}(\theta)$ and its gradient with respect to its Gaussian component $\N_k(\theta) = \N(\theta; m_k, C_k)$ are approximated as 
\begin{equation}
\label{eq:log_rho_a}
\E_{\N_k}[\log \rho_a^{\rm GM}(\theta)] \approx  \log \rho_a^{\rm GM}(m_k) ,\qquad 
\E_{\N_k}[\nabla_\theta \log \rho_a^{\rm GM}(\theta)] \approx  \nabla_\theta \log \rho_a^{\rm GM}(m_k).
\end{equation}
The expectation of the Hessian of $\log \rho_a^{\rm GM}(\theta)$ with respect to its Gaussian component $\N_k$ is approximated as 
\begin{equation}
\label{eq:logrho-Hessian}
\begin{split}
&\E_{\N_k}\Bigl[\nabla_\theta \nabla_\theta \log \rho_a^{\rm GM}(\theta)\Bigr] 
\\
&=  -\E_{\N_k}\Bigl[\frac{\Bigl(\sum_i w_i v_i \N_i(\theta)\Bigr)\Bigl(\sum_i w_i v_i \N_i(\theta)\Bigr)^T}{ \rho_a^{\rm GM}(\theta)^2} 
+\frac{\sum_i w_i \Bigl(C_i^{-1}- v_iv_i^T\Bigr)\N_i(\theta)}{ \rho_a^{\rm GM}(\theta)}\Bigr]
\\
&=  \underbrace{\E_{\N_k}\Bigl[\frac{\sum_{i<j} w_iw_j \bigl(v_i - v_j\bigr)\bigl(v_i - v_j\bigr)^T\N_i(\theta)\N_j(\theta)}{ \rho_a^{\rm GM}(\theta)^2}\Bigr]}_{term \,\, 1}
\underbrace{-\E_{\N_k}\Bigl[\frac{\sum_i w_i   \N_i(\theta) C_i^{-1}}{ \rho_a^{\rm GM}(\theta)}\Bigr]}_{term \,\, 2}
\\
&\approx \frac{\sum_{i<j} w_iw_j \bigl(v_i(m_k) - v_j(m_k)\bigr)\bigl(v_i(m_k) - v_j(m_k)\bigr)^T\N_i(m_k)\N_j(m_k)}{ \rho_a^{\rm GM}(m_k)^2} - C_k^{-1}.
\end{split}
\end{equation}
Here we denote $v_i(\theta) = C_i^{-1}(\theta - m_i)$.
\end{definition}

The expectations of $\log \rho_a^{\rm GM}(\theta)$ and its gradient in \cref{eq:log_rho_a} are computed using the mean-point approximation \eqref{eq:mean_point_stationary}, which is consistent with the approximation used for $\Phi_R$ in \cref{def:DF-quadrature}.
For the Hessian approximation in \cref{eq:logrho-Hessian}, it is decomposed into two terms.
The first term is positive definite and approximated using the mean-point approximation.The second term is negative definite and may cause a loss of covariance positivity in the update, if we simply use a mean-point approximation. We consider the weighted average which leads to the following equality
    \begin{align*}
        -\sum_k w_k \E_{\N_k}\Bigl[\frac{\sum_i w_i   \N_i (\theta) C_i^{-1}}{ \rho_a^{\rm GM}(\theta)}\Bigr] = -\int \sum_i w_i   \N_i (\theta) C_i^{-1}\dd\theta = -\sum_k w_k C_k^{-1}.
    \end{align*}
This shows that the weighted average of the second term is the same as the weighted average of  $-C_k^{-1}$. Based on this fact, we approximate the second term as $-C_k^{-1}$, rather than using a direct mean-point approximation. This approximation can be understood as a correction. When $K=1$, this approximation becomes exact. 
Moreover, we will show later in \cref{prop:positivity} that this correction ensures covariance positivity, significantly enhancing the stability of the algorithm.

With the above quadrature rules, we update the covariances, means, and weights sequentially using a forward Euler scheme from time $t$ to time $t+\Delta t$ with time step $\Delta t$. For the mean update, we use the updated covariance 
$C_{k}(t+\Delta t)$ in a Gauss–Seidel manner on the right hand side.
\begin{subequations}
\label{eq:alg}
\begin{align}
C_{k}^{-1}(t+\Delta t) 
        &= C_{k}^{-1}(t)  + \Delta t ~ \mathrm{\mathsf{QR}}_{\N_k}\bigl\{\nabla_{\theta}\nabla_{\theta}\log \rho_a^{\rm GM}(t) + \nabla_{\theta}\nabla_{\theta}\Phi_R\bigr\}, \label{eq:alg-C}
        \\
        m_{k}(t+\Delta t) 
        &= m_{k}(t)  - \Delta t ~ C_k(t+\Delta t) ~ \mathrm{\mathsf{QR}}_{\N_k} \bigl\{\nabla_{\theta} \log\rho_a^{\rm GM}(t)  +  \nabla_{\theta} \Phi_R \bigr\},
        \label{eq:alg-m}
        \\
        \log w_{k}(t+\Delta t) &= \log w_{k}(t) - \Delta t ~ \mathrm{\mathsf{QR}}_{\N_k} \bigl\{\log \rho_a^{\rm GM}(t)  + \Phi_R \bigr\}. 
        \label{eq:alg-w}
\end{align}
\end{subequations}
Here, $\mathrm{\mathsf{QR}}$ denotes the quadrature rules as defined in \cref{def:DF-quadrature,def:GM-quadrature}, and $\N_k$ denotes the $k$-th Gaussian component at time $t$.
We then normalize  ${w_{k}(t+\Delta t)}_{k=1}^{K}$ and, for efficiency, set a lower bound of 
$w_{k}$ at a default value of $10^{-8}$ during normalization. The lower bound, combined with the logarithmic update rule for weights in \cref{eq:alg-w}, allows weights to recover exponentially from small values, which is advantageous for alleviating potential mode collapse.
The computational cost with respect to evaluating $\F$ is $(2N_\theta + 1) K N_t$. This is because evaluating the quadrature rule in \cref{def:DF-quadrature} requires computing $\F$ at $2N_\theta + 1$ quadrature points for each Gaussian mode, where $N_t$ represents the total number of time steps. It is worth mentioning the evaluation of $\F$ at the $(2N_\theta + 1) K$ quadrature points can be performed in an embarrassingly parallel manner.

\section{Theoretical Insight}
\label{sec:theory}
In this section, we present the theoretical insights underlying the proposed DF-GMVI method, including its \textit{covariance positivity preserving} and \textit{affine invariance} properties. These insights are also applied to the design of adaptive time-stepping schemes and select gradient flows, which significantly improve the performance of other Gaussian mixture variational inference methods (see \cref{ssec:Multi-Dimensional-Problems}).

A common issue with Gaussian mixture-based sampling methods is instability when using large time steps, due to the covariance matrix losing positive-definiteness, as discussed in \cref{sec:guidelines}. The following proposition establishes a condition on the time step $\Delta t$ that ensures the covariance matrix remains positive definite during the update process, which relies on our design of the quadrature rules used for approximating Hessian expectations in \cref{sec:df-gmvi}. 
The proposition guarantees that we can choose a sufficiently large  $\Delta t$ (on the order of 
$\mathcal{O}(1)$), enabling faster convergence without introducing numerical instability (see \cref{sec:numerics}).
\begin{proposition}
\label{prop:positivity}
For the DF-GMVI algorithm described in \cref{eq:alg}, if $0< \Delta t < 1$, then $C_k$ remains positive definite.
\end{proposition}
\begin{proof}
Incorporating the quadrature rules from \cref{eq:Hessian} and \cref{eq:logrho-Hessian} into \cref{eq:alg}, the update rule for $C_k$ becomes 
    \begin{equation}
\begin{split}
C_{k}^{-1}(t+\Delta t) 
        =& C_{k}^{-1}(t)  + \Delta t ~ \mathrm{\mathsf{QR}}_{\N_k}\bigl\{\nabla_{\theta}\nabla_{\theta}\log \rho_a^{\rm GM}(t) + \nabla_{\theta}\nabla_{\theta}\Phi_R\bigr\}
        \\
        =& C_{k}^{-1}(t)  + \Delta t ~ \Bigl(
        \sqrt{C_k(t)}^{-T}(6\mathrm{Diag}(A_k^TA_k) + B_k^T B_k)\sqrt{C_k(t)}^{-1} - C_k^{-1}(t) 
        \\
        &+ \frac{\sum_{i<j} w_iw_j \bigl(v_i(m_k) - v_j(m_k)\bigr)\bigl(v_i(m_k) - v_j(m_k)\bigr)^T\N_i(m_k)\N_j(m_k)}{ \rho_a^{\rm GM}(m_k)^2}
        \Bigr)
        \\
        =& (1 - \Delta t)C_{k}^{-1}(t)  + \Delta t ~ \Bigl(
        \sqrt{C_k(t)}^{-T}(6\mathrm{Diag}(A_k^TA_k) + B_k^T B_k)\sqrt{C_k(t)}^{-1}
        \\
        &+ \frac{\sum_{i<j} w_iw_j \bigl(v_i(m_k) - v_j(m_k)\bigr)\bigl(v_i(m_k) - v_j(m_k)\bigr)^T\N_i(m_k)\N_j(m_k)}{ \rho_a^{\rm GM}(m_k)^2}
        \Bigr).
\end{split}
\end{equation}
Since the matrix within the brackets is always positive semi-definite and $(1 - \Delta t)C_{k}^{-1}(t)$ is positive definite, we conclude that $C_{k}^{-1}(t+\Delta t)$ is also positive definite. 
\end{proof}

Moreover, the following proposition establishes the affine invariance of the natural gradient flow \eqref{eq:Appr-FR-GM}, and a restricted affine invariance of the DF-GMVI method.

\begin{proposition} 
\label{prop:affine}
The sampling algorithm applied to the posterior distribution, which is proportional to $\mathrm{e}^{-\Phi_R(\theta)}$, generates Gaussian mixture approximations parameterized by $\{m_{k}(t), C_{k}(t), w_{k}(t)\}_{k=1}^{K}$. 
For the transformed posterior, with $\widetilde{\Phi}_R(\widetilde{\theta}) = \Phi_R(T^{-1}(\widetilde{\theta} - d))$, under an arbitrary invertible affine mapping \( \varphi : \theta \rightarrow \widetilde{\theta} = T\theta + d \), where $T \in \mathcal{T}$ is an invertible matrix and $d$ is a translation vector, it produces updated parameters $\{\widetilde{m}_{k}(t), \widetilde{C}_{k}(t), \widetilde{w}_{k}(t)\}_{k=1}^K$. The algorithm is $\mathcal{T}$-invariant if the following holds:
\begin{align*}
\widetilde{w}_{k}(t) = w_{k}(t), \quad \widetilde{m}_{k}(t) = T m_{k}(t)+ d, \quad \widetilde{C}_{k}(t) = T C_{k}(t) T^T.
\end{align*}
This implies:
\begin{enumerate}
    \item The natural gradient flow defined by \cref{eq:Appr-FR-GM} is affine invariant, i.e., $\mathcal{T}$ contains all invertible matrices.
    \item The DF-GMVI algorithm \cref{eq:alg} is $\mathcal{T}$-invariant, when using Cholesky decomposition to compute the square root matrix  $\sqrt{C}$ in \cref{def:DF-quadrature},  where $\mathcal{T}$ denotes the group of invertible lower triangular matrices.
\end{enumerate}

\end{proposition}
We provide a proof of this proposition in \cref{sec:proof}. The affine invariance is restricted because the Cholesky decomposition of the covariance is only invariant for the lower triangular matrix group. The affine invariance enables the DF-GMVI method and variational inference methods based on natural gradient flow \eqref{eq:Appr-FR-GM} to exhibit faster convergence than other gradient flows with Gaussian mixture approximations, as further demonstrated numerically in \cref{ssec:Multi-Dimensional-Problems}.

\section{Numerical Study}
\label{sec:numerics}
In this section, we present numerical studies on the proposed DF-GMVI method. We focus on posterior distributions of unknown parameters or fields arising in inverse problems that may exhibit multiple modes. Three types of model problems are considered:
\begin{enumerate}
    \item A one-dimensional bimodal problem from \cite{chen2024efficient}: we use this problem as a proof-of-concept example. Our result shows that the convergence rate remains unchanged no matter how overlapped the two modes are. Moreover, DF-GMVI outperforms the derivative free Gaussian mixture Kalman inversion presented in \cite{chen2024efficient}.
    \item Several multi-dimensional problems: we use these problems to demonstrate that DF-GMVI is robust with respect to the posterior featuring multiple modes, infinitely many modes, and modes with narrow and curved shapes. The convergence rate remains robust across different problem dimensionalities. 
    Additionally, we compare DF-GMVI with other state-of-the-art sampling approaches, including Gaussian mixture variational inference methods leveraging gradient or Hessian information, BBVI, and MCMC methods, to highlight its strengths. 
    Here for these comparable methods, we design adaptive time-stepping schemes, following our guideline in \cref{sec:guidelines} to ensure covariance positivity. This strategy also substantially enhances the robustness of these Gaussian mixture variational inference methods.
    \item Navier-Stokes problem: we consider the inverse problem of recovering the initial velocity field of a Navier-Stokes flow. The problem is structured to exhibit symmetry, resulting in two modes in the posterior distribution. We demonstrate that DF-GMVI effectively captures both modes, highlighting its potential for tackling multimodal problems in large-scale, high-dimensional applications.
\end{enumerate}
For all tests, we set $\Delta t = 0.5$ and run 200 iterations, resulting in a total time of $T = 100$. We set $\alpha = 10^{-3}$ in the derivative free quadrature rule in \cref{def:DF-quadrature}. A parameter sensitivity study is presented in \cref{sec:parameter}. Our code is available online:
\url{https://github.com/PKU-CMEGroup/InverseProblems.jl/tree/master/Derivative-Free-Variational-Inference}.

\subsection{One-Dimensional Problems}
\label{ssec:1d-bimodal}
In this subsection, we consider the 1D bimodal inverse problem from \cite{chen2024efficient}, associated with the forward model
$$y = \G(\theta) + \eta \quad \textrm { with } \quad y = 1, \quad \G(\theta) = \theta^2.$$
The prior is $\rho_{\rm prior} \sim \N(3, 2^2)$ and examine different noise levels:
\begin{align*}
\textrm{Case A:} \quad  \eta \sim \N(0, 0.2^2);  \\
\textrm{Case B:} \quad  \eta \sim \N(0, 0.5^2); \\
\textrm{Case C:} \quad  \eta \sim \N(0, 1.0^2); \\
\textrm{Case D:} \quad  \eta \sim \N(0, 2.0^2).
\end{align*}
Note that the overlap between these two modes increases as the noise level rises. In case A, the two modes are well separated, while in case D, they are nearly indistinguishable (see \cref{fig:1D-density-all}).

We apply DF-GMVI with $K = 2$, $10$, and $40$ modes, each initialized by randomly sampling from the prior distribution with uniform weights. The converged density estimations at the $200$th iteration are shown in \cref{fig:1D-density-all}, with each row corresponding to a case (A to D). The reference density is plotted in gray. For $K = 10$ and $40$, the estimated densities closely align with the reference. The fourth panel in each row displays the total variation error at each iteration, calculated as $\int |\rho - \rho_{\rm ref}| d\theta$. Although Gaussian mixture variational inference involves non-convex optimization with local minimizers, increasing $K$ generally improves performance. In contrast, the GMKI method proposed in \cite{chen2024efficient} fails to accurately capture the density in the modal overlap regions, similar to the failure caused by inconsistent approximation discussed in \cref{sec:guidelines}. Therefore, our approach outperforms that method while maintaining the same computational cost and derivative free property.

\begin{figure}[ht]
\centering
    \includegraphics[width=0.95\textwidth]{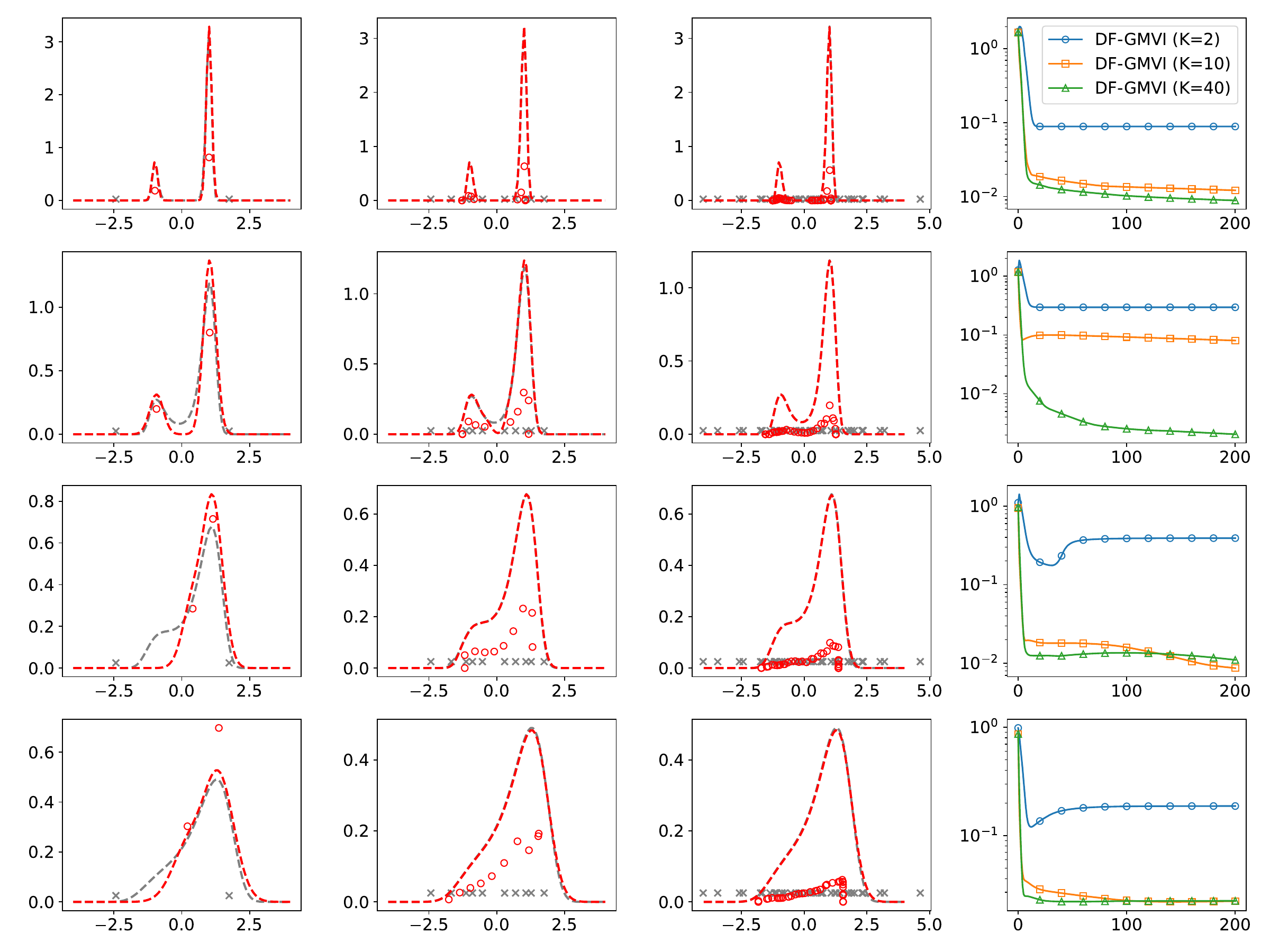}
    \caption{
    Results for the 1D bimodal problem with $\Sigma_{\eta}$ values of $0.2^2$~(top), $0.5^2$~(top middle), $1.0^2$~(bottom middle), and $1.5^2$~(bottom) is shown. Each panel displays the reference posterior distribution (gray dashed lines) alongside the posterior distributions estimated by DF-GMVI (red dashed lines) with mode numbers $K=2,\,10,\,40$~(from left to right).  The mean $m_k$ of each Gaussian component is marked by red circles, and the initial means are indicated by gray crosses, with their $y$-values corresponding to the respective weights (demonstrating no mode collapse). The fourth figure shows the total variation distance between the reference posterior distribution and the DF-GMVI estimated posteriors at each iteration.
    }
    \label{fig:1D-density-all}
\end{figure}

\subsection{Multi-Dimensional Problems}
\label{ssec:Multi-Dimensional-Problems}
In this subsection, we first investigate several 2D sampling problems, along with their 100-dimensional modified versions. We then conclude with a comprehensive comparison to other state-of-the-art sampling approaches. For the 2D sampling problems, we define $\Phi_R(\theta) = \frac{1}{2} \F(\theta)^T \F(\theta)$, where $\theta = [\theta_{(1)}, \theta_{(2)}]^T \in \mathbb{R}^2$, and $\F(\theta)$ is defined as follows.
\begin{enumerate}[label=Case \Alph*, left=0.2cm]
\item : The distribution is a Gaussian with
\begin{align*}
    \F(\theta) =   y - A\theta \quad A= 
\begin{bmatrix}
1 & 1\\
1 & 2
\end{bmatrix} ,\quad \textrm{ and }\quad 
y= 
\begin{bmatrix}
0\\
1
\end{bmatrix}. 
\end{align*}
\item : The distribution has four modes with different weights \cite[Appendix D]{chen2024efficient}.
\begin{align*}
\quad  \F(\theta) = y - \begin{bmatrix}
(\theta_{(1)} - \theta_{(2)})^2\\ 
(\theta_{(1)} + \theta_{(2)})^2\\ 
\theta_{(1)}\\ 
\theta_{(2)}
\end{bmatrix} 
\quad \textrm{ and }\quad 
y= 
\begin{bmatrix}
4.2297\\ 
4.2297\\ 
0.5\\ 
0.0
\end{bmatrix}.
\end{align*}
\item : The distribution is circular in shape \cite[Appendix E]{chen2024efficient} and contains infinitely many modes, as also used in \cref{sec:guidelines}, 
\begin{align*}
\F(\theta) = \frac{y - (\theta_{(1)}^2 + \theta_{(2)}^2)}{0.3}
\quad \textrm{ and }\quad 
\quad 
y = 1.
\end{align*}
\item : 
The distribution follows the Rosenbrock function, which has a characteristic ``banana'' shape~\cite{goodman2010ensemble}, with
\begin{align*}
\F(\theta) = 
\frac{1}{\sqrt{10}}\Bigl(y - \begin{bmatrix}
10(\theta_{(2)} -  \theta_{(1)}^2)\\
\theta_{(1)}
\end{bmatrix}\Bigr)
\quad \textrm{ and }\quad 
y= 
\begin{bmatrix}
0\\
1
\end{bmatrix}.
\end{align*}
\item : The distribution is an extension of the Rosenbrock function \cite{detommaso2018stein} and is characterized by its bimodal, ``banana'' shaped density.
\begin{align*}
\F(\theta) = 
y - \begin{bmatrix}
\log\bigl(
100(\theta_{(2)} -  \theta_{(1)}^2)^2 + (1 - \theta_{(1)})^2
\bigr)/0.3\\
\theta_{(1)}\\
\theta_{(2)}
\end{bmatrix}
\quad \textrm{ and }\quad  
y= 
\begin{bmatrix}
\log(101)\\
0\\
0
\end{bmatrix}. 
\end{align*}
\end{enumerate}

We apply DF-GMVI with $K = 10, 20$ and $40$ modes, each randomly initialized as $\N(0, I)$ with equal weights. The density estimations at the 200th iteration are shown in \cref{fig:2D-density-all}, with each row corresponding to a case (A to E). In each row, the reference density is displayed in the first panel, followed by the results from DF-GMVI with $K = 10, 20$ and $40$ modes. The estimated densities closely match the reference. The fifth panel in each row shows the error in terms of total variation over the iterations.  DF-GMVI effectively captures multiple modes (Cases B and E), densities characterized by narrow, curved valleys (Cases D and E), which typically pose challenges for sampling methods, as well as densities with infinitely many maximum a posteriori points (Case C). The fifth panel demonstrates that DF-GMVI converges efficiently among all cases.
However, it is important to note that when the posterior is Gaussian (Case A), DF-GMVI does not achieve exponential convergence due to the interactions among different Gaussian components; 
in contrast, Gaussian variational inference \cite{chen2023sampling} achieves exponential convergence in this setting. Furthermore, when the number of modes $K$ is small, or when the initial Gaussian components are distant from the target mode, DF-GMVI may fail to capture all modes of the posterior density. Increasing $K$ helps mitigate this issue and enhances performance.

\begin{figure}[h]
\centering
    \includegraphics[width=0.95\textwidth]{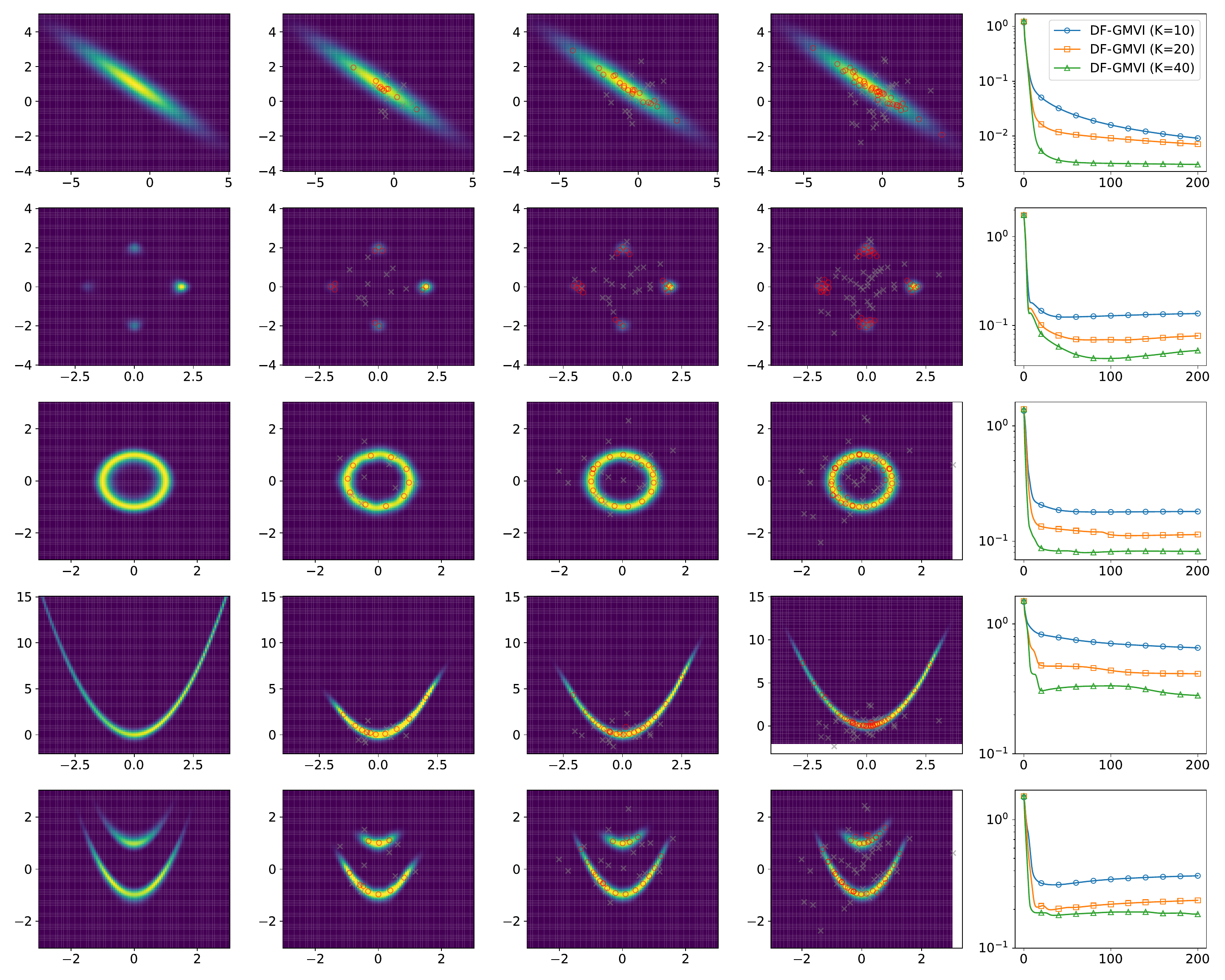}
    \caption{Multi-dimensional problems with dimensionality $2$, arranged from Case A to Case E from the top row to the bottom row. 
    Each panel shows the reference density alongside the densities estimated by DF-GMVI with $K=10, 20$ and $40$ (from left to right). The mean $m_k$ of each Gaussian component is marked by red circles, and the initial means are marked by gray crosses. The fifth panel shows the total variation distance between the reference density and the densities estimated by the DF-GMVI over the iterations.}
    \label{fig:2D-density-all}
\end{figure}

We then modify these sampling problems to a $100$-dimensional problem ($N_\theta = 100$) by introducing $N_\theta - 2$ additional variables, $\theta^c$. The reference density is defined with
\begin{equation*}
    \Phi_R(\theta,  \theta^c) = \frac{1}{2}\F(\theta)^T\F(\theta) + \frac{1}{2}(\theta^c - K\theta)^T (\theta^c - K\theta), 
\end{equation*}
where $\theta \in \mathbb{R}^2$, $\theta^c \in \mathbb{R}^{N_\theta - 2}$, and $K \in \mathbb{R}^{(N_\theta - 2) \times 2}$ is an all-ones matrix. The function $\F$ is defined as in the previous 2-dimensional densities. These high-dimensional densities are constructed so that the marginal densities of $\theta$ preserve computability and match exactly that of the previous 2-dimensional densities, allowing for exact error evaluation through 2D marginal densities to evaluate DF-GMVI.

Similarly, we apply DF-GMVI with $K = 10, 20$ and $40$ modes, each randomly initialized as $\N(0, I)$ with equal weights. The marginal densities obtained from the estimated 100-dimensional densities at the 200th iteration are shown in \cref{fig:100D-density-all}, using the same configuration as in \cref{fig:2D-density-all}. The estimated marginal densities closely align with the reference, and the total variation distances are similar to those from the 2-dimensional problems in \cref{fig:2D-density-all}. This confirms that the performance of DF-GMVI remains robust across different dimensionalities. We conjecture that the difficulty of the sampling problem depends primarily on the complexity of the target density (e.g., the number of modes) rather than the dimensionality.

\begin{figure}[h]
\centering
    \includegraphics[width=0.95\textwidth]{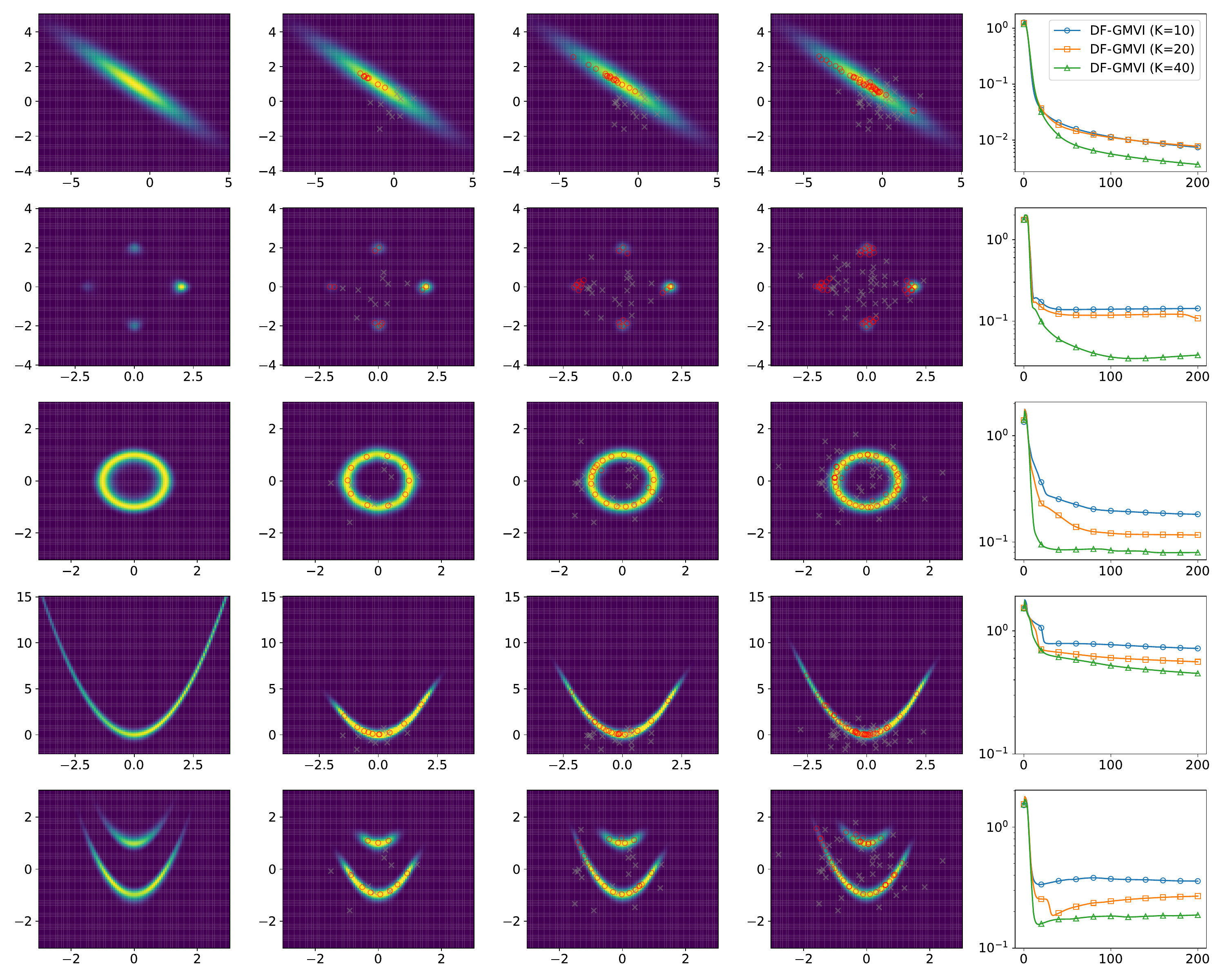}
    \caption{
    Multi-dimensional problems with dimensionality 100, arranged from Case A to Case E from the top row to the bottom row. Each panel shows the marginal density of the reference alongside the marginal densities estimated by DF-GMVI with $K = 10, 20$, and $40$ (from left to right). The projected means of each Gaussian component are marked by red circles, and the projected initial means are marked by gray crosses. The fifth panel displays the total variation between the reference marginal density and the estimated marginal densities across the iterations.
    }
    \label{fig:100D-density-all}
\end{figure}

Finally, we compare DF-GMVI with other state-of-the-art methods, including the Gaussian mixture approximation of the natural gradient flow (NGF-VI)~\cref{eq:Appr-FR-GM} \cite{lin2019fast}, its diagonal covariance variant (NGF-VI-D)~\cite{nguyen2023wasserstein}, the Wasserstein gradient flow (WGF-VI)\cite{lambert2022variational}, BBVI~\cite{ranganath2014black}, and the affine invariant MCMC method \cite{goodman2010ensemble,foreman2013emcee}. 
NGF-VI, NGF-VI-D, and WGF-VI are not derivative free, requiring the computation of the gradient or Hessian of $\Phi_R$. Expectations are approximated using the mean-point approximation \cref{eq:mean_point}. 
Motivated by the covariance positivity insight outlined in \cref{sec:theory}, for the NGF-VI and NGF-VI-D, we implement an adaptive time-stepping scheme defined as
$$\Delta t = \min\{\Delta t_{\max}, \frac{\beta}{\max_k \lVert C_k \E_{\N_k}[\nabla_{\theta}\nabla_{\theta}\log \rho_a^{\rm GM}(t) + \nabla_{\theta}\nabla_{\theta}\Phi_R]\rVert_2}\}$$ 
with $\Delta t_{\max}=0.5$ and $\beta=0.99$, ensuring that the covariance update $C_{k}^{-1}(t+\Delta t) 
        = C_{k}^{-1}(t)  + \Delta t \E_{\N_k}[\nabla_{\theta}\nabla_{\theta}\log \rho_a^{\rm GM}(t) + \nabla_{\theta}\nabla_{\theta}\Phi_R]$ 
maintains positive definiteness. The adaptive time-stepping for NGF-VI significantly improves efficiency compared to the fixed time step used in \cref{sec:guidelines}.
For the WGF-VI, the covariance update equation can be written as
$$\dot{C}_{k}^{-1}
        = C_{k}^{-1}\mathbb{E}_{\N_k} +\mathbb{E}_{\N_k} C_{k}^{-1},\quad
    \E_{\N_k}=\E_{\N_k}[\nabla_{\theta}\nabla_{\theta}\log \rho_a^{\rm GM}(t) + \nabla_{\theta}\nabla_{\theta}\Phi_R].$$
We follow the update method proposed in \cite{lambert2022variational},
$C_k^{-1}(t+\Delta t)=M(t)C_k^{-1}(t)M(t)$, $M(t)=I+\Delta t\mathbb{E}_k(t)$, 
where $I$ is the identity matrix, to achieve first-order accuracy while ensuring covariance positivity. The time step $\Delta t$ is manually chosen to ensure stability, with values $1.4\times 10^{-1}$, $5.0\times 10^{-3}$, $5.0\times 10^{-3}$, $4.0\times 10^{-3}$ and $8.0\times 10^{-4}$ for Case A to Case E, respectively. 
For the BBVI, we adopt the natural gradient flow \eqref{eq:Appr-FR-GM}, derive the mean and covariance update equations by using integration by parts, as follows:
\begin{equation}
\begin{split}
\label{eq:bbvi}
        \dot{m}_{k} 
        &= -\E_{\N_k}[ \bigl(\theta - m_k \bigr)\bigl(\log\rho_a^{\rm GM} + \Phi_R\bigr)],
        \\
        \dot{C}_{k} 
        &= 
        -\E_{\N_k}\Bigl[\bigl( (\theta - m_k)(\theta - m_k)^T - C_k\bigr)(\log \rho_a^{\rm GM}  + \Phi_R)\Bigr].
\end{split}
\end{equation}
We apply the Monte Carlo method with $5$ ensembles (same as DF-GMVI) to compute these integrals. We use empirical covariance instead of $C_k$ in the right hand side of the covariance update equation to reduce variance, and implement a similar adaptive time-stepping scheme defined as
$$\Delta t = \min\{\Delta t_{\max}, \frac{\beta}{\max_k \lVert \E_{\N_k}\bigl[\bigl( (\theta - m_k)(\theta - m_k)^T - C_k\bigr)(\log \rho_a^{\rm GM}  + \Phi_R)\bigr] C_k^{-1}\rVert_2}\}$$ 
with $\Delta t_{\max}=0.5$ and $\beta=0.99$.
Finally, for the affine invariant MCMC method, we use the stretch move method proposed in \cite{goodman2010ensemble}.
To recover the equilibrium distributions, we apply Gaussian kernel density estimation with a selectively chosen bandwidth using particles sampled from the last 10 iterations.

\Cref{fig:MultiModal-Comparison-2D} shows the results obtained from Gaussian mixture variational inference approaches with $K=40$ components and the affine-invariant MCMC method using $J=10^3$ particles, applied to 2-dimensional problems over 500 iterations. We observe that these variational inference approaches require very small time steps, particularly in Case E, to ensure stability. Our proposed adaptive time steppers significantly improve this. 
The BBVI method exhibits random noise, which can be mitigated by increasing the ensemble size for Monte Carlo integrals. Among the variational inference methods, those based on natural gradient flows with affine invariance outperform their counterparts based on Wasserstein gradient flows. Overall,  the results obtained by DF-GMVI, as shown in \cref{fig:2D-density-all,fig:MultiModal-Comparison-2D}, are similarly comparable and avoid the use of derivatives.
In high dimensional spaces, the volume grows exponentially, which makes it difficult for random derivative free methods to explore the space efficiently. As a result, the performance of BBVI and derivative free MCMC methods typically deteriorates with increasing dimensionality, as illustrated in \cref{fig:MultiModal-Comparison-100D}. However, the quadrature rule \cref{def:DF-quadrature} leverages gradient structure, allowing our DF-GMVI method to avoid this issue even in the 100-dimensional problem.
Consequently, our DF-GMVI method outperforms the other approaches on the test problems presented in this subsection.

\begin{figure}[ht]
    \centering
    \includegraphics[width=0.99\linewidth]{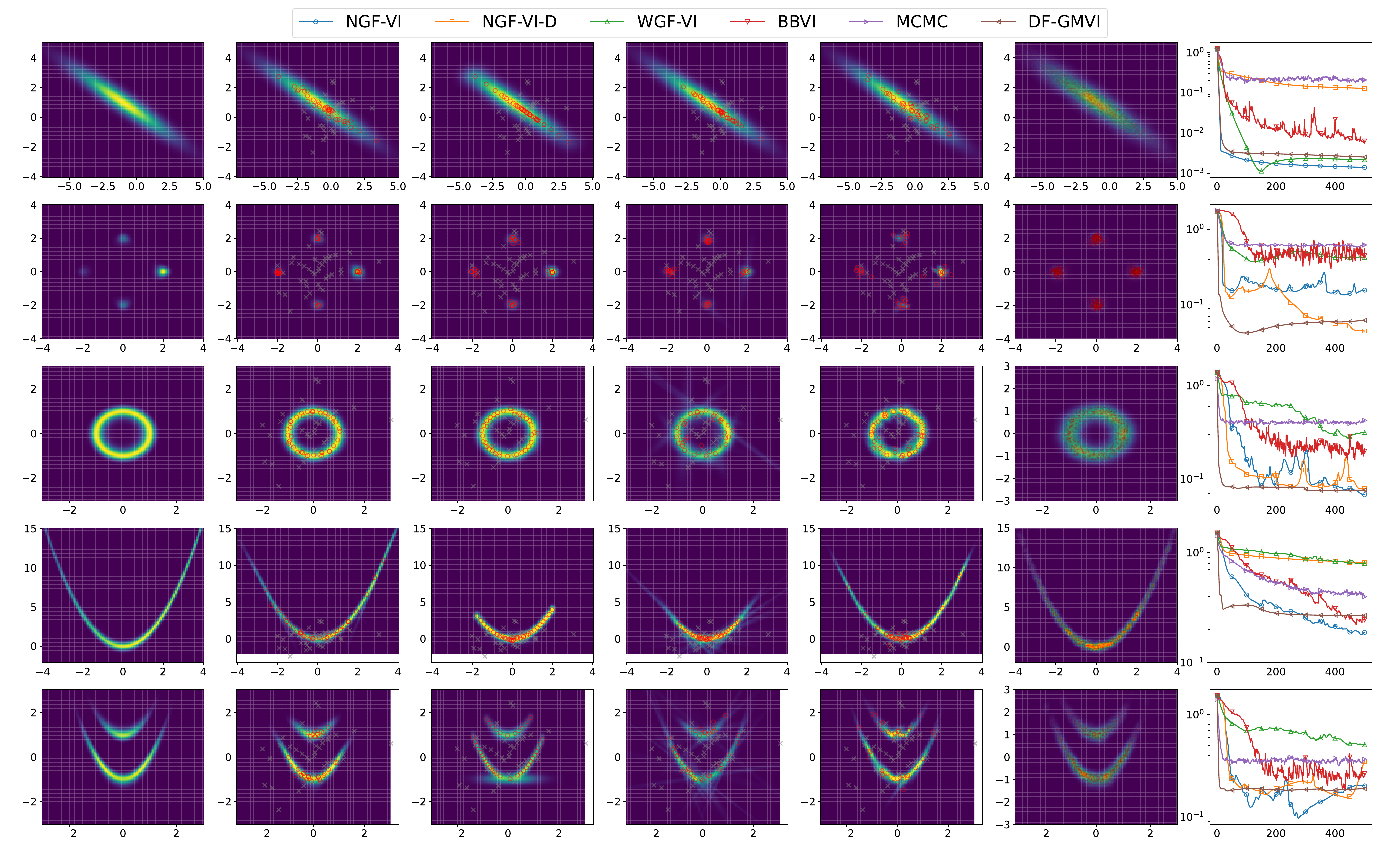}
    \caption{Comparison study for multi-dimensional problems with dimensionality $2$, arranged from Case A to Case E (from top to bottom). Each panel displays the reference density alongside the densities estimated by variational inference methods and the affine invariant MCMC~(left to right: NGF-VI, NGF-VI-D, WGF-VI, BBVI, MCMC). 
    The last figure shows the total variation distance between the reference density and the densities estimated by these methods and DF-GMVI over the iterations.
    }
    \label{fig:MultiModal-Comparison-2D}
\end{figure}

\begin{figure}[ht]
    \centering
    \includegraphics[width=0.95\linewidth]{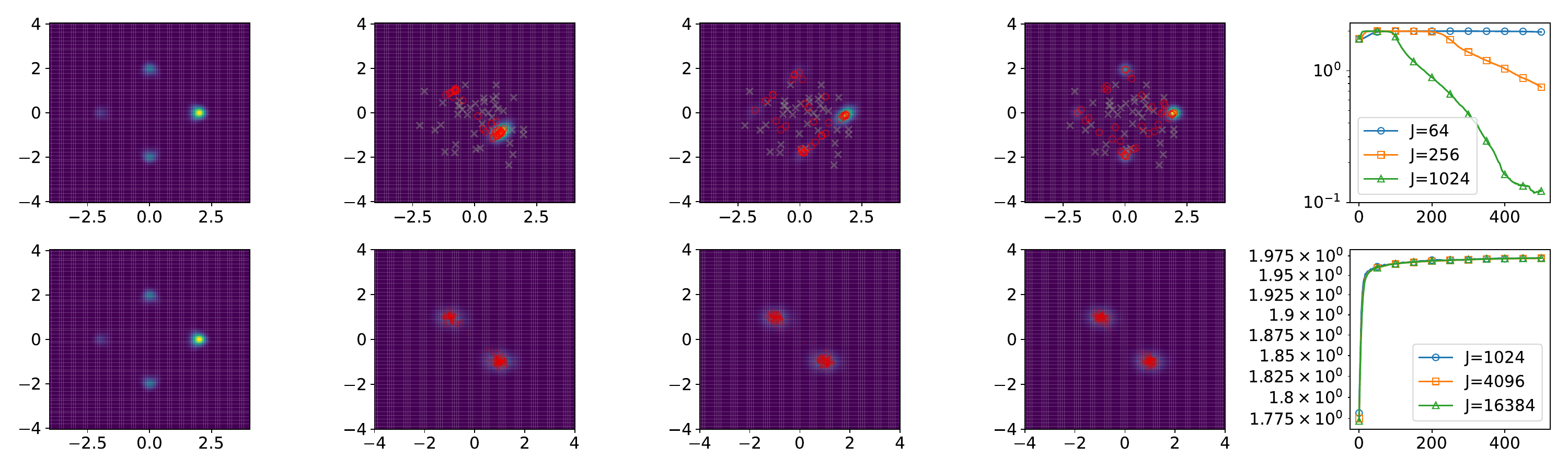}
    \caption{Comparison study for multi-dimensional problems with dimensionality $100$ (Case B).
    The figures display the marginal reference density and the estimated marginal densities obtained by BBVI with $K=40$ and ensemble sizes $J=64, 256$, and $1024$ for Monte Carlo integration, and by affine invariant MCMC method with ensemble sizes $J=2^{10}, 2^{12}$, and $2^{14}$ (from left to right) at the $500$-th iteration.
    The last figure shows the total variation distance between the marginal reference density and the estimated marginal densities over iterations.}  
    \label{fig:MultiModal-Comparison-100D}
\end{figure}

\subsection{High-Dimensional Inverse Problem}

In this subsection, we study the problem of recovering the initial vorticity field, $\omega_0$, of a fluid flow based on measurements taken at later times. The flow is governed by the 2D incompressible Navier-Stokes equation on a periodic domain $D = [0,2\pi]\times[0,2\pi]$, written in the vorticity-streamfunction~$\omega-\psi$ formulation:
\begin{equation}
\label{eq:NS}
\begin{split}
    &\frac{\partial \omega}{\partial t} + (v\cdot\nabla)\omega - \nu\Delta\omega = \nabla \times f, \\
    &\omega = -\Delta\psi \qquad \frac{1}{4\pi^2}\int\psi = 0
    \qquad v = [\frac{\partial \psi}{\partial x_2}, -\frac{\partial \psi}{\partial x_1}]^T + v_b.
\end{split}
\end{equation}
Here $v$ denotes the velocity vector, $\nu=0.01$ denotes the viscosity, $v_b = [0, 2\pi]^T$ denotes the non-zero mean background velocity, and $f(x) = [0,\cos(4x_{(1)})]^T$ denotes the external forcing.

The problem setup follows \cite{chen2024efficient}, which is spatially symmetric with respect to $x_{(1)} = \pi$. The source of the fluid is chosen such that 
$\nabla \times f([x_{(1)},x_{(2)}]^T) = - \nabla \times f([2\pi-x_{(1)},x_{(2)}]^T). $
The observations in the inverse problem are chosen as the difference of pointwise measurements of the vorticity value 
$\omega([x_{(1)}, x_{(2)}]^T) - \omega([2\pi - x_{(1)}, x_{(2)}]^T)$
at $35$ equidistant points in the lower-left region of the domain~(see \Cref{fig:NS-2d-ref}), at $T=0.25$ and $T=0.5$, corrupted with observation error $\eta \sim \N(0, 0.3^2\I)$.  
Under this set-up, both $\omega_0([x_{(1)},x_{(2)}]^T)$ and $-\omega_0([2\pi-x_{(1)},x_{(2)}]^T)$ will lead to the same measurements. Thus the posterior of the Bayesian inverse problem will be at least bi-modal.

We assume the prior of $\omega_0(x, \theta)$ is a Gaussian field with covariance operator $C = (-\Delta)^{-2}$, subject to periodic boundary conditions, on the space of mean zero functions. 
Following \cite{hairer2006ergodicity}, let $\mathbb{Z}_{+}^2 = \{(l_1, l_2)\in\mathbb{Z}^2, l_2 > 0 \}\cup\{ (l_1, 0)\in\mathbb{Z}^2, l_1 > 0\}$ and $\mathbb{Z}_{-}^2 = -\mathbb{Z}_{+}^2$, the corresponding KL expansion of the initial vorticity field is given by 
\begin{equation}
\label{eq:NS-KL-2d}
\omega_0(x, \theta) = \sum_{l\in \mathbb{Z}^2 \backslash \{(0,0)\}} \theta_{(l)} \sqrt{\lambda_{l}} \psi_l(x),
\end{equation}
where eigenpairs are of the form
\begin{equation*}
    \psi_l(x) = \begin{cases}
        \frac{\sin(l\cdot x)}{\sqrt{2}\pi} & l\in \mathbb{Z}_{+}^2\\
        \frac{\cos(l\cdot x)}{\sqrt{2}\pi} & l\in \mathbb{Z}_{-}^2
    \end{cases},\quad \lambda_l = \frac{1}{|l|^{4}},
\end{equation*}
and $\theta_{(l)} \sim \N(0,2\pi^2)$. The KL expansion~\cref{eq:NS-KL-2d} can be rewritten as a sum over positive integers, 
where the eigenvalues $\lambda_l$ are in descending order. 
We truncate the expansion to the first $128$ terms and generate the truth vorticity field $\omega_0(x; \theta_{\rm ref})$ with $\theta_{\rm ref} \in \R^{128}$; we aim to recover the parameters based on observation data. Therefore, the problem is ill-posed, as the number of unknown parameters exceeds the amount of observed data.

\begin{figure}[ht]
\centering
    \includegraphics[width=0.4\textwidth]{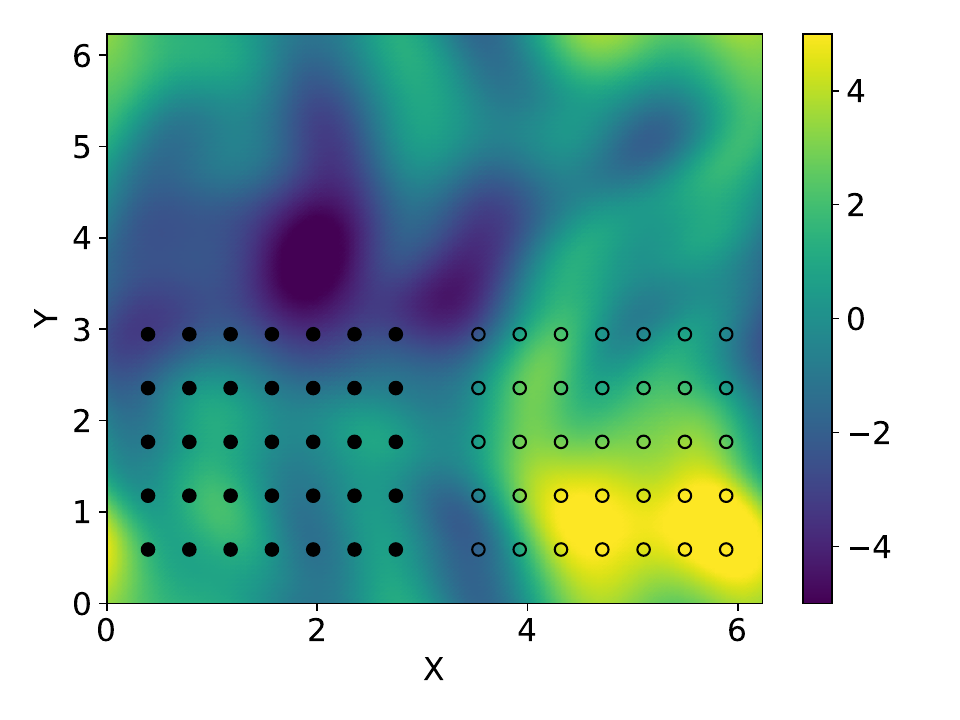}
    \includegraphics[width=0.4\textwidth]{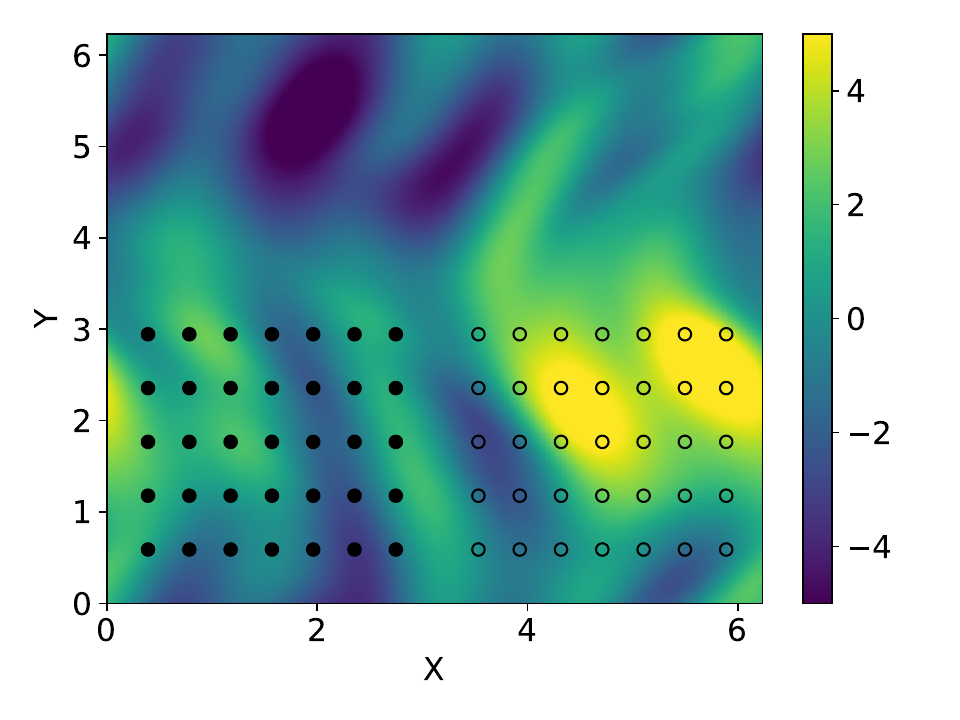}
    \caption{The vorticity field $\omega$ at $T=0.25$ and $T=0.5$ and observations $\omega([x_{(1)}, x_{(2)}]^T) - \omega([2\pi - x_{(1)}, x_{(2)}]^T)$ at $56$ equidistant points (solid black dots). Their mirroring points are marked (empty black dots).}
    \label{fig:NS-2d-ref}
\end{figure}
We apply DF-GMVI with $K = 5$ modes, initialized randomly according to the prior distribution and assigned equal weights.  In \cref{fig:NS-2D-vor}, we display the true initial vorticity field, $\omega_0(x; \theta_{\rm ref})$, along with its mirrored field (where the mirroring of the velocity field induces antisymmetry in the vorticity field), and the five recovered initial vorticity fields, $\omega_0(x; m_k)$, obtained by DF-GMVI at the 200th iteration. Modes 1 and 2 capture $\omega_0(x; \theta_{\rm ref})$, while modes 3 to 5 capture the mirrored field of $\omega_0(x; \theta_{\rm ref})$ itself.
\Cref{fig:NS-2D-convergence} illustrates the relative errors in the vorticity field, optimization errors \( \Phi_R(m_{k}) \), the Frobenius norm \( \lVert C_{k} \rVert_F \), and the Gaussian mixture weights \( w_{k} \) (left to right) at each iteration. This demonstrates that DF-GMVI converges in fewer than 50 iterations.
\Cref{fig:NS-2D-density} shows the marginal distributions of the estimated posterior for the first 16 $\theta_{(l)}$-coefficients obtained by DF-GMVI. The marginal distributions reveal clear bimodality, with the approximate posterior placing high probability mass around the true coefficients while also capturing the alternative possibility (the mirrored value). This example demonstrates DF-GMVI’s potential for effectively handling multimodal posteriors in large-scale, high-dimensional applications.

\begin{figure}[ht]
\centering
    \includegraphics[width=\textwidth]{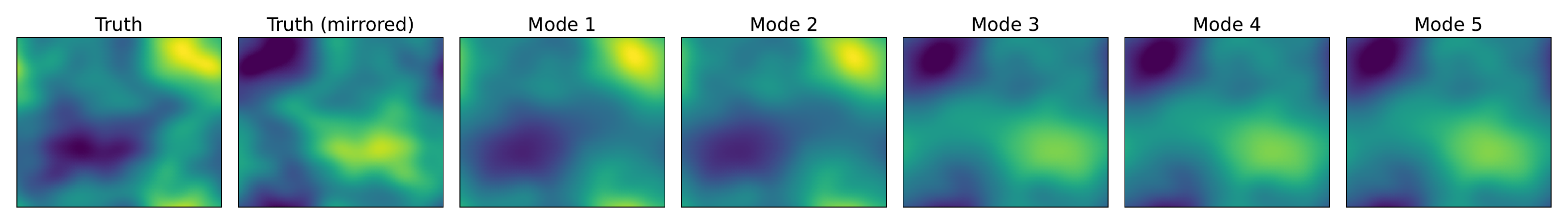}
    \caption{The true initial vorticity field $ \omega_0(x; \theta_{\rm ref})$, and recovered initial vorticity fields $\omega_0(x; m_k)$ obtained by DF-GMVI.}
    \label{fig:NS-2D-vor}
\end{figure}

\begin{figure}[ht]
\centering
    \includegraphics[width=\textwidth]{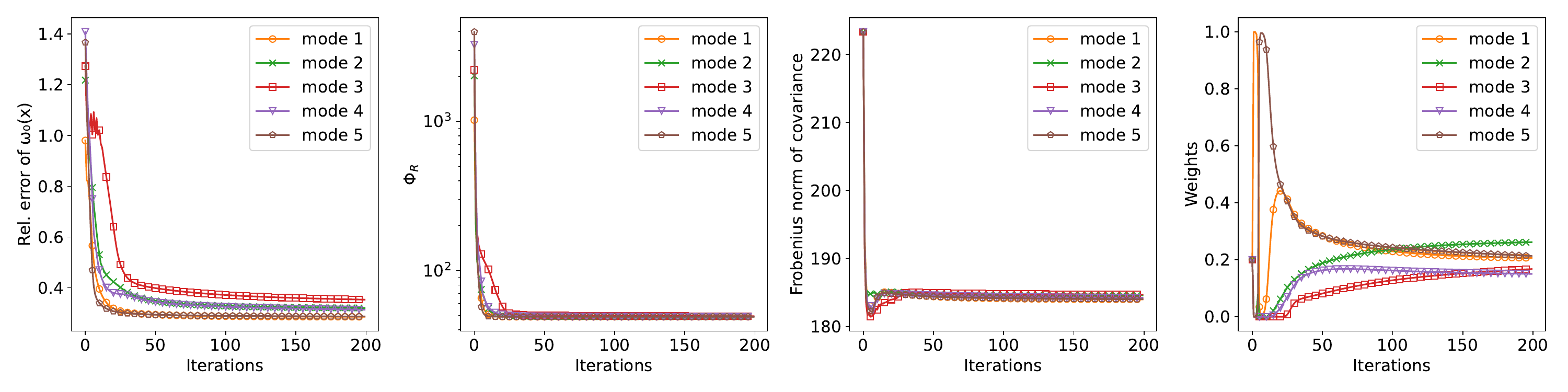}
    \caption{Navier-Stokes flow problem: the relative errors of the initial vorticity field, the optimization errors $\Phi_R(m_{k})$, the Frobenius norm $\lVert C_{k}\rVert_F$, and the Gaussian mixture weights $w_{k}$ (from left to right) for different modes over iterations.}
    \label{fig:NS-2D-convergence}
\end{figure}

\begin{figure}[ht]
\centering
    \includegraphics[width=0.9\textwidth]{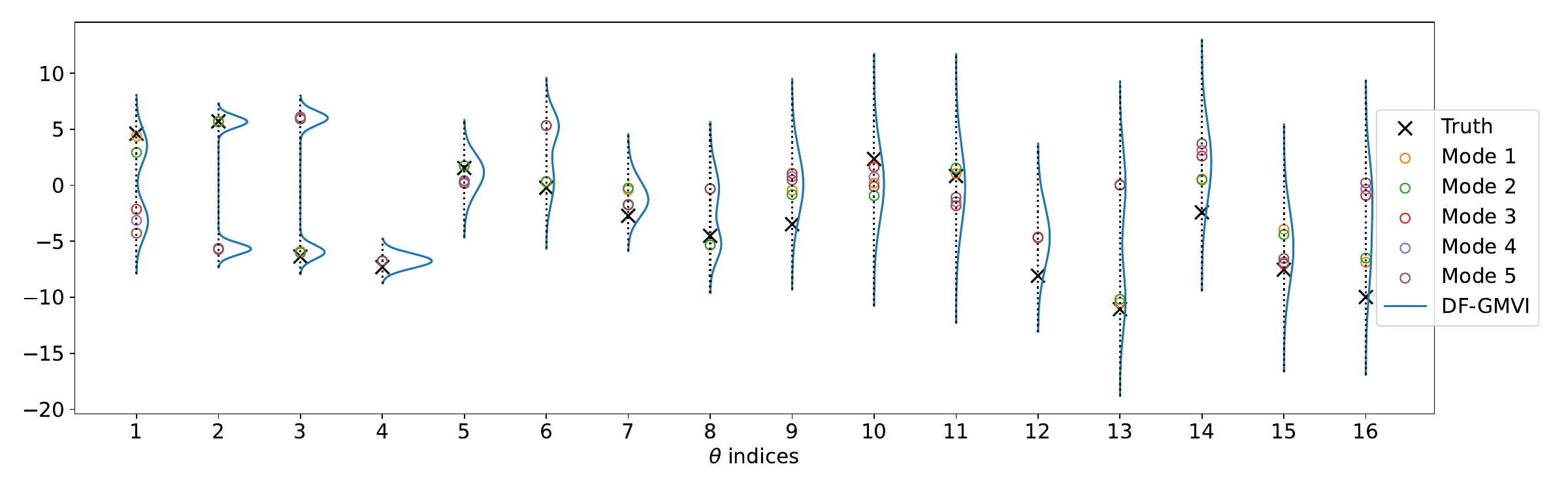}
    \caption{Navier-Stokes flow problem: the true Karhunen-Loeve expansion parameters $\theta_{(i)}$ (black crosses),  and mean estimations of $\theta_{(i)}$ for each modes~(circles) and the associated marginal distributions obtained by DF-GMVI at the $50$th iteration.}
    \label{fig:NS-2D-density}
\end{figure}

\section{Conclusion}
\label{sec:conclusion}
In this paper, we have explored guidelines, such as covariance positivity, to design stable Gaussian mixture variational inference methods. Based on these guidelines,  we have introduced  DF-GMVI as  an effective derivative free approach for solving Bayesian inverse problems. Our numerical experiments demonstrate the superior stability of DF-GMVI and its capability to approximate complex high dimensional posterior distributions.
There are several promising directions for future research.
On the algorithmic side,  several directions are promising for further improving practical applicability: developing stable derivative free variational inference methods beyond Bayesian inverse problems (see \cref{eq:PhiR}); extending the quadrature approximation defined in \cref{def:DF-quadrature} to higher-order schemes; and exploring adaptive strategies for determining the number of Gaussian components $K$.
On the theoretical side, a thorough analysis of the stability and convergence properties of the natural gradient flow \cref{eq:Appr-FR-GM} with Gaussian mixtures, as well as DF-GMVI for general posterior densities, including log-concave densities, could provide insights for their practical applications.

\section*{Acknowledgments} 
The authors acknowledge the support of the National Natural Science Foundation of China through grant 12471403, the Fundamental Research Funds for the Central Universities, and the high-performance computing platform of Peking University.
\bibliographystyle{siamplain}
\bibliography{references}

\appendix

\section{Proof of \cref{prop:affine}}
\label{sec:proof}
\begin{proof}
We note that
\begin{align*}
    \nabla_{\theta}\Phi_R(\theta)= T^T\nabla_{\widetilde{\theta}}
    \widetilde{\Phi}_R(\widetilde{\theta}),\quad 
    \nabla_{\theta}\nabla_{\theta}\Phi_R(\theta)= T^T\nabla_{\widetilde{\theta}}\nabla_{\widetilde{\theta}}\widetilde{\Phi}_R(\widetilde{\theta})T. 
\end{align*}
    
\begin{enumerate}
\item To prove the affine invariance of the natural gradient flow defined by \cref{eq:Appr-FR-GM}, it suffices to verify that under the affine mapping
\begin{align*}
    \widetilde{w}_{k} = w_{k}, \quad \widetilde{m}_{k} = T m_{k} + d, \quad \widetilde{C}_{k} = T C_{k} T^T,
\end{align*}
$\widetilde{w}_{k}, \widetilde{m}_{k}$ and $\widetilde{C}_{k}$ also satisfy the corresponding equation \cref{eq:Appr-FR-GM}.
We simplify the notation by $\widetilde{\N}_k(\widetilde{\theta})=\N(\widetilde{\theta},\widetilde{m_k},\widetilde{C_k}),\,\rho_{\widetilde{a}}^{\mathrm{GM}}(\widetilde{\theta})=\sum_{k=1}^K\widetilde{w}_{k}\widetilde{\N}_k(\widetilde{\theta})$. We have $\widetilde{\N}_k(\widetilde{\theta})=|\det T|^{-1}\N_k(\theta),\,\rho_{\widetilde{a}}^{\mathrm{GM}}(\widetilde{\theta})=|\det T|^{-1}\rho_a^{\mathrm{GM}}(\theta)$, and thus $\nabla_{\widetilde{\theta}}\log\rho_{\widetilde{a}}^{\rm GM}(\widetilde{\theta})=T^{-T} \nabla_{\theta}\log\rho_{a}^{\rm GM}(\theta)$, $\nabla_{\widetilde{\theta}}\nabla_{\widetilde{\theta}}\log\rho_{\widetilde{a}}^{\rm GM}(\widetilde{\theta}) =T^{-T} \nabla_{\theta}\nabla_{\theta}\log\rho_{a}^{\rm GM}(\theta)T^{-1}$. By these equations and change-of-variable formula, we have that
\begin{align*}
    \dot{\widetilde{m}}_{k}=T\dot{m}_k    
    &=-T C_k \int \N_k(\theta) \Bigl(\nabla_{\theta}\log\rho_{a}^{\rm GM}(\theta)+\nabla_\theta \Phi_R(\theta) \Bigr)\dd \theta\\
    &=-T C_k T^T\int \widetilde{\N}_k(\widetilde{\theta}) T^{-T}\Bigl(\nabla_{\theta}\log\rho_{a}^{\rm GM}(\theta)+\nabla_\theta \Phi_R (\theta) \Bigr)\dd \widetilde{\theta}\\
    &=-\widetilde{C}_k\int \widetilde{\N}_k(\widetilde{\theta}) \Bigl( \nabla_{\widetilde{\theta}} \log\rho_{\widetilde{a}}^{\rm GM}(\widetilde{\theta})  +  \nabla_{\widetilde{\theta}} \widetilde{\Phi}_R (\widetilde{\theta}) \Bigr)  \dd \widetilde{\theta},\\
    \dot{\widetilde{C}}_{k}=T\dot{C}_kT^T    
    &=-T C_k \Bigr( \int \N_k(\theta) \bigl(\nabla_{\theta}\nabla_{\theta}\log\rho_{a}^{\rm GM}(\theta)+\nabla_\theta \nabla_{\theta} \Phi_R(\theta) \bigr)\dd \theta \Bigr)C_k T^T  \\
    &=-T C_k T^{T} \Bigr( \int \widetilde{\N}_k(\widetilde{\theta})  T^{-T}\bigl(\nabla_{\theta}\nabla_{\theta}\log\rho_{a}^{\rm GM}(\theta)+\nabla_\theta \nabla_{\theta}\Phi_R(\theta) \bigr) T^{-1} \dd \widetilde{\theta} \Bigr)T C_k T^T \\
    &=-\widetilde{C}_k\Bigl(\int \widetilde{\N}_k(\widetilde{\theta})\bigl(\nabla_{\widetilde{\theta}}\nabla_{\widetilde{\theta}}\log \rho_{\widetilde{a}}^{\rm GM}(\widetilde{\theta})  + \nabla_{\widetilde{\theta}}\nabla_{\widetilde{\theta}}\widetilde{\Phi}_R(\widetilde{\theta})\bigr) \dd\widetilde{\theta}\Bigr) \widetilde{C}_k,\\
    \dot{\widetilde{w}}_{k}=\dot{w}_k    
    &=-w_k\int \Bigl(\N_k(\theta) -  \rho_a^{\rm GM}(\theta)\Bigr)\bigl(\log \rho_a^{\rm GM}(\theta)  + \Phi_R(\theta) \bigr) \dd\theta \\
    &=-\widetilde{w}_k\int \Bigl(\widetilde{\N}_k(\widetilde{\theta}) -  \rho_{\widetilde{a}}^{\rm GM}(\widetilde{\theta})\Bigr)\bigl(\log |\det T|+\log \rho_{\widetilde{a}}^{\rm GM}(\widetilde{\theta})  + \widetilde{\Phi}_R(\widetilde{\theta}) \bigr) \dd\widetilde{\theta} \\
    &=-\widetilde{w}_k\int \Bigl(\widetilde{\N}_k(\widetilde{\theta}) -  \rho_{\widetilde{a}}^{\rm GM}(\widetilde{\theta})\Bigr)\bigl(\log \rho_{\widetilde{a}}^{\rm GM}(\widetilde{\theta})  + \widetilde{\Phi}_R(\widetilde{\theta}) \bigr) \dd\widetilde{\theta},
\end{align*}
where in the last equality, we used the fact that $\widetilde{\N}_k$ and $\rho_{\widetilde{a}}^{\rm GM}$ are both densities and thus $\int \bigl(\widetilde{\N}_k(\widetilde{\theta}) -  \rho_{\widetilde{a}}^{\rm GM}(\widetilde{\theta})\bigr)\log |\det T| \dd\widetilde{\theta}=0$. This completes the proof.
\item To prove that the DF-GMVI algorithm \cref{eq:alg} is $\mathcal{T}$-invariant, it suffices to show that the quadrature rules defined in \cref{def:DF-quadrature} and \cref{def:GM-quadrature} satisfy the following equations for any $T\in\T$:
\begin{equation}
    \label{eq:affine-trans-QR-aim}
    \begin{aligned}
    \mathrm{\mathsf{QR}}_{\widetilde{\N}_k}\bigl\{\nabla_{\widetilde{\theta}}\nabla_{\widetilde{\theta}}\log \rho_{\widetilde{a}}^{\rm GM} + \nabla_{\widetilde{\theta}}\nabla_{\widetilde{\theta}}\widetilde{\Phi}_R\bigr\}
    &=T^{-T}\mathrm{\mathsf{QR}}_{\N_k}\bigl\{\nabla_{\theta}\nabla_{\theta}\log \rho_a^{\rm GM} + \nabla_{\theta}\nabla_{\theta}\Phi_R\bigr\}T^{-1},\\
    \mathrm{\mathsf{QR}}_{\widetilde{\N}_k} \bigl\{\nabla_{\widetilde{\theta}} \log\rho_{\widetilde{a}}^{\rm GM} + \nabla_{\widetilde{\theta}} \widetilde{\Phi}_R \bigr\}
    &=T^{-T}\mathrm{\mathsf{QR}}_{\N_k} \bigl\{\nabla_{\theta} \log\rho_a^{\rm GM}  +  \nabla_{\theta} \Phi_R \bigr\}, \\
    \mathrm{\mathsf{QR}}_{\widetilde{\N}_k} \bigl\{\log \rho_{\widetilde{a}}^{\rm GM} + \widetilde{\Phi}_R \bigr\}
    &=\mathrm{\mathsf{QR}}_{\N_k} \bigl\{\log \rho_a^{\rm GM}  + \Phi_R \bigr\}-\log|\det T|.
    \end{aligned}
\end{equation}
Here in the last equality, we can tolerate a constant difference since we will normalize ${w_{k}(t+\Delta t)}_{k=1}^{K}$ after the update \cref{eq:alg}. Given any Gaussian density $\N=\N(\theta;m,C)$, an affine mapping transforms it to another Gaussian $\widetilde{\N}=\N(\widetilde{\theta};\widetilde{m},\widetilde{C})$. The square root matrix computed by Cholesky decomposition satisfies $\sqrt{\widetilde{C}}=\sqrt{T C T^T}=T\sqrt{C}$, when  $T$ is an invertible lower triangular matrix. Hence, the quadrature points generated for the two Gaussian densities according to \cref{def:DF-quadrature} satisfy
\begin{align*}
    \widetilde{\theta}_0&
    =\widetilde{m}=Tm+d=T\theta_0+d,\\
    \widetilde{\theta}_i&=\widetilde{m}+\alpha [\sqrt{\widetilde{C}}]_i=Tm+d+\alpha T[\sqrt{C}]_i=T\theta_i+d,\\
    \widetilde{\theta}_{N_\theta+i}&=\widetilde{m}-\alpha [\sqrt{\widetilde{C}}]_i=Tm+d-\alpha T[\sqrt{C}]_i=T\theta_{N_\theta+i}+d,
\end{align*}
for $1\leq i\leq N_\theta$. These equations imply that $\widetilde{\F}(\widetilde{\theta_i})=\F(\theta_i)$ for $0\leq i\leq2N_\theta$, and thus $\widetilde{c}=c$, $\widetilde{b_i}=b_i$, $\widetilde{a_i}=a_i$ for $1\leq i\leq N_\theta$, and then $\widetilde{B}=B$, $\widetilde{A}=A$. Consequently,
\begin{equation}
    \label{eq:affine-trans-QR-1}
    \begin{aligned}
    \mathrm{\mathsf{QR}}_{\widetilde{\N}}\{\widetilde{\Phi}_R\}
    &=\frac{1}{2}\widetilde{c}^T\widetilde{c}=\frac{1}{2}c^Tc= \mathrm{\mathsf{QR}}_{\N}\{\Phi_R\},\\
    \mathrm{\mathsf{QR}}_{\widetilde{\N}}\{\nabla_{\widetilde{\theta}}\widetilde{\Phi}_R\}
    &=\sqrt{\widetilde{C}}^{-T}\widetilde{B}^T\widetilde{c}
    =T^{-T}\sqrt{C}^{-T}B^Tc=T^{-T} \mathrm{\mathsf{QR}}_{\N} \{\nabla_\theta \Phi_R\},\\
    \mathrm{\mathsf{QR}}_{\widetilde{\N}}\{\nabla_{\widetilde{\theta}}\nabla_{\widetilde{\theta}}\widetilde{\Phi}_R\}
    &=\sqrt{\widetilde{C}}^{-T}(6\mathrm{Diag}(\widetilde{A}^T\widetilde{A}) + \widetilde{B}^T \widetilde{B})\sqrt{\widetilde{C}}^{-1}\\
    &=T^{-T}\sqrt{C}^{-T}(6\mathrm{Diag}(A^TA) + B^T B)\sqrt{C}^{-1}T^{-1} =T^{-T}\mathrm{\mathsf{QR}}_{\N} \{\nabla_\theta\nabla_\theta \Phi_R \}T^{-1}.
    \end{aligned}
\end{equation}
By the \cref{def:GM-quadrature}, we have that
\begin{equation}
    \label{eq:affine-trans-QR-2}
    \begin{aligned}
    &\mathrm{\mathsf{QR}}_{\widetilde{\N}_k}\{\log\rho_{\widetilde{a}}^{\rm GM}\}
    =\log \rho_{\widetilde{a}}^{\rm GM}(\widetilde{m}_k)
    =\log \left( |\det T|^{-1} \rho_{a}^{\rm GM}(m_k) \right)
    =\mathrm{\mathsf{QR}}_{\N_k}\{\log\rho_a^{\rm GM}\}-\log|\det T|,
    \\
    &\mathrm{\mathsf{QR}}_{\widetilde{\N}_k}\{\nabla_{\widetilde{\theta}}\log\rho_{\widetilde{a}}^{\rm GM}\}
    =\nabla_{\widetilde{\theta}} \log \rho_{\widetilde{a}}^{\rm GM}(\widetilde{m}_k)
    =T^{-T}\nabla_{\theta}\log \rho_{a}^{\rm GM}(m_k)
    =T^{-T}\mathrm{\mathsf{QR}}_{\N_k}\{\nabla_{\theta}\log\rho_a^{\rm GM}\}.
    \end{aligned}
\end{equation}
Using the same notations as in \cref{def:GM-quadrature}, we have $\widetilde{v_i}(\widetilde{\theta})=\widetilde{C}_i^{-1}(\widetilde{\theta} - \widetilde{m_i})=(TC_iT^T)^{-1}(T\theta - T m_i )=T^{-T}v_i(\theta)$, and thus it is easy to check that
\begin{equation}
    \label{eq:affine-trans-QR-3}
    \mathrm{\mathsf{QR}}_{\widetilde{\N_k}}\{\nabla_{\widetilde{\theta}}\nabla_{\widetilde{\theta}}\log\rho_{\widetilde{a}}^{\rm GM}\}
    =T^{-T}\mathrm{\mathsf{QR}}_{\N_k}\{\nabla_{\theta}\nabla_{\theta}\log\rho_a^{\rm GM}\}T^{-1}.
\end{equation}
Combining \cref{eq:affine-trans-QR-1}, \cref{eq:affine-trans-QR-2} and \cref{eq:affine-trans-QR-3}, we finish the proof of \cref{eq:affine-trans-QR-aim}.
\end{enumerate}
\end{proof}

\section{Parameter Sensitivity Study}
\label{sec:parameter}
In this section, we study the sensitivity of DF-GMVI on hyperparameters $\alpha$ and $\Delta t$.
Theoretically, the finite difference approximation in \cref{def:DF-quadrature} exhibits low sensitivity to relatively small $\alpha$, and \cref{prop:positivity} guarantees covariance positive definiteness for all $\Delta t \in (0,1)$.
Empirically, we test the sensitivity by running DF-GMVI (K=40) for 200 iterations on Case E in \cref{ssec:Multi-Dimensional-Problems}, with different choices of $\alpha$ and $\Delta t$. 
To test the sensitivity on $\alpha$, we fix $\Delta t =0.5$ and assign $\alpha = 10^{-1}, 10^{-3}$ and $10^{-5}$. 
To test the sensitivity on $\Delta t$, we fix $\alpha = 10^{-3}$ and assign $\Delta t = 0.25, 0.50$ and $0.75$. 
The results are shown in \cref{fig:DF-GMVI-param-compare}, which demonstrate that the performance of DF-GMVI remains stable across different settings of parameters $\alpha$ (relatively small) and $\Delta t\in(0,1)$. We believe that the users can fix $\alpha=10^{-3}$ and $\Delta t=0.5$ for different Bayesian inverse problems.

\begin{figure}[h]
    \centering
    \includegraphics[width=0.7\linewidth]{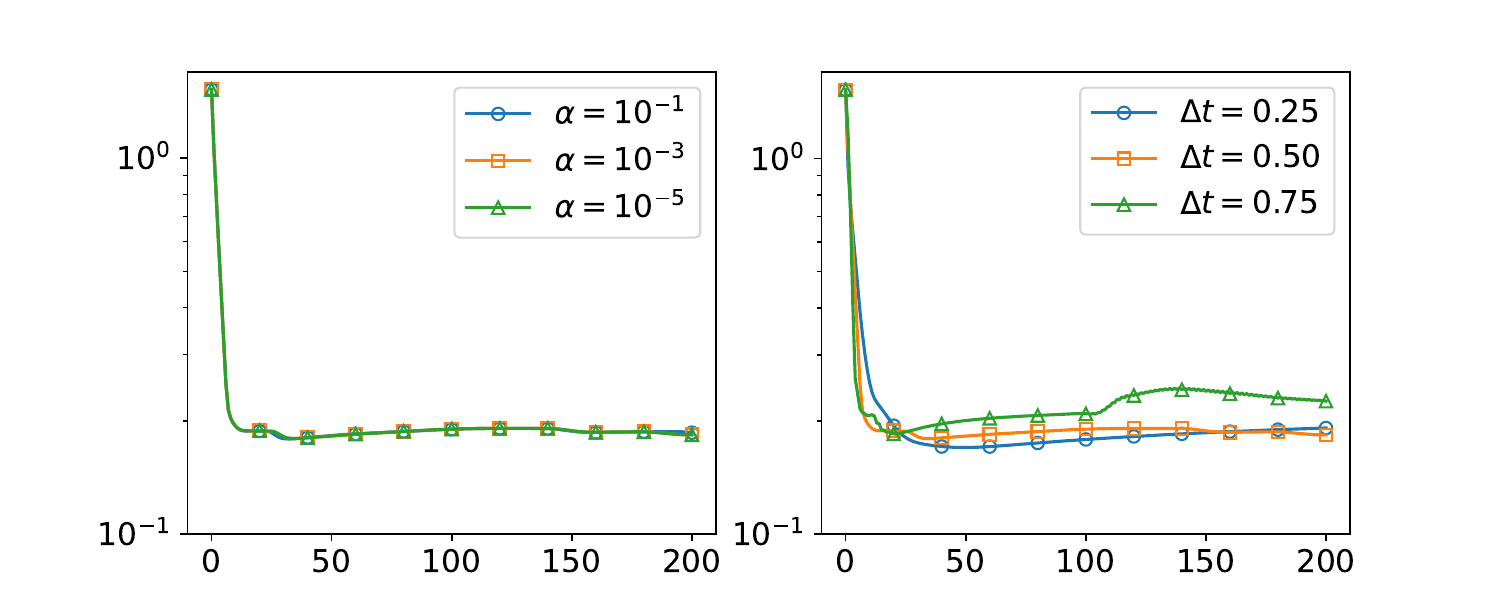}
    \caption{Comparison for different choices of $\alpha$ and $\Delta t$ on the Case E of multi-dimensional problems with dimensionality 2. The left figure shows the total variation between the reference density and the densities estimated by the DF-GMVI over the iterations using different $\alpha \in \{10^{-1},10^{-3},10^{-5}\}$ with $\Delta t=0.5$, and the right figure shows the total variation using different $\Delta t\in\{0.25,0.50,0.75\}$ with $\alpha=10^{-3}$.}
    \label{fig:DF-GMVI-param-compare}
\end{figure}

\section{Supplement Material}

\subsection{Natural Gradient Variational Inference}

We review natural gradient variational inference from the perspective of gradient flow. Then we discuss both the Gaussian variational family and the Gaussian mixture variational family in Sections \ref{ssec:VI-G} and \ref{ssec:VI-GM}, respectively. Variational inference aims to approximate the posterior distribution by minimizing the Kullback–Leibler (KL) divergence~\cite{wainwright2008graphical,blei2017variational}
\begin{equation}
\label{eq:KL-s}
    {\rm KL}[\rho_a \Vert \rho_{\rm post}] = \int \rho_a \log\Bigl(\frac{\rho_a}{\rho_{\rm post}}\Bigr)\dd\theta
\end{equation}
over a variational family of densities $\rho_a$, parameterized by $a\in \R^{N_a}$. 
When employing gradient descent, taking the continuous time limit, the parameter $a$ evolves according to the gradient flow:
\begin{align} 
    \frac{\dd a}{\dd t} = -\nabla_a  {\rm KL}[\rho_a \Vert \rho_{\rm post}].
\end{align}
The steepest descent direction can be interpreted as 
\begin{equation} 
     -\nabla_a  {\rm KL}[\rho_a \Vert \rho_{\rm post}] = \argmin_\sigma \frac{\langle \nabla_a{\rm KL}[\rho_{a} \Vert \rho_{\rm post}] , \sigma\rangle}{\sqrt{\langle\sigma,  \sigma \rangle}} , 
\end{equation}
where the numerator denotes the descent quantity along $\sigma$ and the denominator denotes the length of $\sigma$ under the Euclidean inner-product $\langle\cdot, \cdot \rangle$ in $\R^{N_a}$.  
When a more general metric, induced by the inner product $\langle\cdot, \fM(\rhoa) \cdot \rangle$ with metric tensor $\fM(\rhoa)$ is used, the steepest descent direction becomes
\begin{equation} 
     -\fM(\rhoa)^{-1}\nabla_a  {\rm KL}[\rho_a \Vert \rho_{\rm post}] = \argmin_\sigma \frac{\langle \nabla_a{\rm KL}[\rho_{a} \Vert \rho_{\rm post}] ,  \sigma\rangle}{\sqrt{\langle\sigma,  \fM(\rhoa) \sigma \rangle}}.  
\end{equation}
This modification leads to a different gradient flow for updating $a$ as
\begin{align} 
    \frac{\dd a}{\dd t} = -\fM(\rhoa)^{-1}\nabla_a  {\rm KL}[\rho_a \Vert \rho_{\rm post}].
\end{align}

The present work focuses on the natural gradient~\cite{amari1998natural},  where the metric tensor is the Fisher information matrix~\cite{rao1945information}
\begin{equation*}
    \fM(\rhoa) = {\rm FIM}(a):= \int  \frac{\partial \log \rho_{\rhoa}(\theta)}{\partial \rhoa} \frac{\partial \log \rho_{\rhoa}(\theta)}{\partial \rhoa}^T \rho_\rhoa(\theta) \dd \theta.
\end{equation*}
The Fisher-information matrix is related to the Hessian matrix of the KL-divergence~\cite{james2020new},
because the Taylor expansion of the KL-divergence~\eqref{eq:KL-s} between $\rho_\rhoa$ and its neighbor $\rho_{\rhoa + \dd\rhoa}$, gives
\begin{equation} 
    {\rm KL}[\rho_{\rhoa + \dd\rhoa}\Vert \rho_{\rhoa}] = \frac{1}{2} \dd\rhoa ^T {\rm FIM}(a)  \dd\rhoa + o(\lVert da\rVert^2).
\end{equation}
Here we used the fact that ${\rm KL}[\rho_{\rhoa}\Vert \rho_{\rhoa}] = 0$ and $\nabla_a {\rm KL}[\rho_{\rhoa}\Vert \rho_{\rhoa}] = 0$ (the gradient is with respect to the first $\rho_a$).
Preconditioning the gradient with the Fisher information matrix inherently incorporates geometric information.
Consequently, akin to Newton's method, the Fisher information matrix finds extensive application as a preconditioner to accelerate the optimization process in variational inference.
This gives rise to  the concept of natural gradient or natural gradient variational inference~\cite{amari1998natural,lin2019fast}, which corresponds to the following gradient flow:
\begin{align} 
\label{eq:NGF}
    \frac{\dd a}{\dd t} = -{\rm FIM}(a)^{-1}\nabla_a  {\rm KL}[\rho_a \Vert \rho_{\rm post}].
\end{align}
Therefore, once the variational family is specified, discretizing the gradient flow introduced above enables the derivation of various practical sampling methods. In what follows, we provide two concrete examples.

\subsubsection{Gaussian Approximation}
\label{ssec:VI-G}
Gaussian variational inference operates within a Gaussian parametric space,  where the variational family \[\rho_a^{\rm G}(\theta) = \N(\theta; m, C)\] represents a Gaussian parameterized by its mean 
$m \in \R^{N_\theta}$ and covariance $C\in \R^{N_\theta \times N_\theta}$, collectively denoted by the parameter vector $a:=[m,C]$.  The Fisher information matrix of Gaussian is 
\begin{equation}
\label{eq:Gaussian-FIM}
{\rm FIM}(a) = 
\begin{bmatrix}
    C^{-1} & \\
    & X
\end{bmatrix},
\end{equation}
where $X$ is a 4-th order tensor. Its action on any matrix $Y \in \R^{N_{\theta}\times N_{\theta}}$  is given by
\begin{equation}
X Y =     \frac{1}{4}C^{-1} (Y + Y^T) C^{-1}.
\end{equation}
We write down the calculation of Fisher information matrix here. First we have that
\begin{align*}
    \int (\nabla_m \log\rho_a^{\rm G})^T(\nabla_m \log\rho_a^{\rm G})\rho_a^{\rm G}\dd \theta &= C^{-1}\Big(\int(\theta-m)^T(\theta-m)\rho_a^{\rm G}\dd \theta\Bigr)C^{-1}=C^{-1},\\
    \int (\nabla_m \log\rho_a^{\rm G})^T(\nabla_C \log\rho_a^{\rm G})\rho_a^{\rm G}\dd \theta &= 0.
\end{align*}
For the covariance item, define a 4-th order tensor 
\begin{align*}
    X=\int \big(\frac{1}{2}C^{-1}(\theta-m)(\theta-m)^T C^{-1}-\frac{1}{2}C^{-1}\bigr)\otimes \big(\frac{1}{2}C^{-1}(\theta-m)(\theta-m)^T C^{-1}-\frac{1}{2}C^{-1}\bigr)\rho_a^{\rm G} \dd \theta.
\end{align*}   
By change of variable $y=C^{-\frac{1}{2}}(\theta-m)$,  we have
\begin{align*}
        X &=\frac{1}{4}\int \left(C^{-\frac{1}{2}}(yy^T-I)C^{-\frac{1}{2}}\right)\otimes \left(C^{-\frac{1}{2}}(yy^T-I)C^{-\frac{1}{2}}\right)\N(y;0,I) \dd y,\\
        X[ij,kl] &=\frac{1}{4}\int \sum_{p,q} C^{-\frac{1}{2}}[i,p](y_py_q-\delta_{p,q})C^{-\frac{1}{2}}[q,j] \sum_{r,s} C^{-\frac{1}{2}}[k,r](y_ry_s-\delta_{r,s})C^{-\frac{1}{2}}[s,l] \N(y;0,I) \dd y \\
        &=\frac{1}{4} \sum_{p,q,r,s} C^{-\frac{1}{2}}[i,p]C^{-\frac{1}{2}}[q,j]C^{-\frac{1}{2}}[k,r] C^{-\frac{1}{2}}[s,l] \int (y_py_q-\delta_{p,q})(y_ry_s-\delta_{r,s})\N(y;0,I) \dd y\\
        &=\frac{1}{4} \sum_{p,q,r,s} C^{-\frac{1}{2}}[i,p]C^{-\frac{1}{2}}[q,j]C^{-\frac{1}{2}}[k,r] C^{-\frac{1}{2}}[s,l] (\delta_{r,p}\delta_{s,q}+ \delta_{r,q}\delta_{s,p})\\
        &=\frac{1}{4} \left(C^{-1}[ik] C^{-1}[jl]+C^{-1}[il] C^{-1}[jk]\right).
\end{align*}
For all $Y\in\R^{N_\theta\times N_\theta}$, the act of $X$ on $Y$ is 
\begin{align*}
        (XY)_{ij} &=\sum_{k,l} X[ij,kl]Y[kl]\\
        &=\frac{1}{4}\sum_{k,l} (C^{-1}[ik] C^{-1}[jl]+C^{-1}[il] C^{-1}[jk])Y[kl]\\
        &=\frac{1}{4} (C^{-1}YC^{-1}+C^{-1}Y^T C^{-1})_{ij},
\end{align*}
and thus $XY=\frac{1}{4}C^{-1}(Y+Y^T)C^{-1}$.

Incorporating the Fisher information matrix \cref{eq:Gaussian-FIM} into \cref{eq:NGF} yields the following natural gradient flow
\begin{equation}
\begin{aligned}
\label{eq:Gaussian Fisher-Rao}
\dot  m_t = - C_t\E_{\rho_{a_t}^{\rm G}}[\nabla \Phi_R ], \qquad
\dot  C_t = C_t - C_t \E_{\rho_{a_t}^{\rm G}}[\nabla^2  \Phi_R ]C_t .
\end{aligned}
\end{equation}
By using the fact 
$
  \dot{C}_t^{-1}
   = -C_t^{-1}\frac{\dd  C_t }{\dd  t} C_t^{-1},
$
we can rewrite the covariance evolution equation as 
\begin{equation}
\label{eq:dC_t^-1}
\begin{aligned}
\dot{C}_t^{-1}
&=  \E_{\rho_{a_t}^{\rm G}}[\nabla^2  \Phi_R ] - C_t^{-1}.
\end{aligned}
\end{equation}

The above dynamics \cref{eq:Gaussian Fisher-Rao} is affine invariant~\cite[Section 5.4.1]{chen2023sampling}. Consequently, when the posterior is Gaussian, it converges exponentially fast to the posterior at a rate of $\mathcal{O}(e^{-t})$~\cite[Theorem 5.6]{chen2023sampling}\cite{garbuno2020interacting,carrillo2021wasserstein,burger2023covariance}, where the exponent of the convergence rate is independent of the posterior. Furthermore, the time discretization of the natural gradient flow \cref{eq:Gaussian Fisher-Rao} exhibits superior stability. This stability arises from the fact that when the posterior density is log-concave, i.e., when $\nabla^2  \Phi_R$ is positive semidefinite, the forward Euler discretization of \cref{eq:dC_t^-1} with $0 < \Delta t < 1$  ensures covariance positivity:
\begin{align*}
    C_{t+\Delta t}^{-1}  =  C_{t}^{-1}  + \Delta t \bigl(\E_{\rho_{a_t}^{\rm G}}[\nabla^2  \Phi_R ] - C_t^{-1}\bigr) = (1 - \Delta t)C_{t}^{-1}  + \Delta t \E_{\rho_{a_t}^{\rm G}}[\nabla^2  \Phi_R ] \succ 0,
\end{align*}
provided that the approximation of $\E_{\rho^{\rm G}_{a_t}}[\nabla^2  \Phi_R ]$ remains positive semidefinite.

\subsubsection{Gaussian Mixture Approximation}
\label{ssec:VI-GM}
Gaussian mixture variational inference considers the Gaussian mixture parametric space,  where the variational family \[\rho_a^{\rm GM}(\theta) = \sum_{k=1}^{K} w_k \N(\theta; m_k, C_k)\] is a $K$-component Gaussian mixture, parameterized by means $m_k\in\R^{N_\theta}$, covariances $C_k\in\R^{N_\theta \times N_\theta}$ and weights $w_k \in \mathbb{R}_{\geq 0}$, collectively denoted by the parameter vector \[a:=[m_1, ..., m_k, ..., m_K, C_1. ..., C_k, ..., C_K, w_1, ..., w_k, ..., w_K].\] 
Weights satisfy $\sum_{k=1}^{K} w_k = 1$.
To compute the gradient flow in \cref{eq:NGF}, we first evaluate the derivatives of the KL divergence in \cref{eq:KL-s} with respect to $a$:
\begin{subequations}
\label{eq:partial_KL_partial_a}
\begin{align}
&\frac{\partial {\rm KL}[\rho_a^{\rm GM} \Vert \rho_{\rm post}]}{\partial m_k} = w_k\int \N_k(\theta) \Bigl( \nabla_{\theta} \log\rho_a^{\rm GM}  +  \nabla_{\theta} \Phi_R \Bigr)  \dd\theta,
\\
& \frac{\partial{\rm KL}[\rho_a^{\rm GM} \Vert \rho_{\rm post}]}{\partial C_k} = \frac{w_k}{2}\int \N_k(\theta) \Bigl(\nabla_{\theta}\nabla_{\theta}\log \rho_a^{\rm GM}  + \nabla_{\theta}\nabla_{\theta} \Phi_R \Bigr) \dd\theta,
\\
&\frac{\partial{\rm KL}[\rho_a^{\rm GM} \Vert \rho_{\rm post}]}{\partial w_k} = \int \N_k(\theta) \Bigl(\log\frac{\rho_a^{\rm GM}}{\rho_{\rm post}} + 1\Bigr) \dd\theta.
\end{align}
\end{subequations}
Here, we simplify the notation by denoting $\N_k(\theta)$ as $\N(\theta; m_k, C_k)$. The steepest descent direction is determined by the following constrained optimization problem:
\begin{equation}
\label{eq:descent-direction-GM}
     \argmin_\sigma \frac{\langle \nabla_a{\rm KL}[\rho_{a}^{\rm GM} \Vert \rho_{\rm post}] , \sigma\rangle}{\sqrt{\langle\sigma,  {\rm FIM}(a)\sigma \rangle}}   \quad \textrm{ s.t. } \quad \sum_{k=1}^{K}\sigma_{\dot{w}_k} = 0,
\end{equation}
where $\{\sigma_{\dot{w}_k}\}$ represents the descent directions corresponding to the weights. We introduce a Lagrangian multiplier $\lambda \in \R$ and write down the Karush–Kuhn–Tucker conditions of \cref{eq:descent-direction-GM}:
\begin{equation*}
    \nabla_d\Big(\frac{\langle\nabla_a {\rm KL}(\rho_a^{\rm GM}||\rho_{\rm post}),d\rangle}{\sqrt{\langle d, {\rm FIM}(a)d\rangle}}\Bigr) + \lambda\nabla_d \Big(\sum_{k=1}^K d_{w_k}\Bigr)=0,\quad \sum_{k=1}^K d_{w_k}=0,
\end{equation*}
where the stability condition yields
\begin{equation}
    \label{eq: kkt}
    \begin{split}
    \langle \nabla_a {\rm KL}[\rho_{a}^{\rm GM} \Vert \rho_{\rm post}], d\rangle\, [{\rm FIM}(a)d]_{m_k} &=
    \langle d, {\rm FIM}(a)d\rangle\, \nabla_{m_k}  {\rm KL}[\rho_{a}^{\rm GM} \Vert \rho_{\rm post}],
   \\
   \langle \nabla_a {\rm KL}[\rho_{a}^{\rm GM} \Vert \rho_{\rm post}], d\rangle\, [{\rm FIM}(a)d]_{C_k} &=
   \langle d, {\rm FIM}(a)d\rangle\, \nabla_{C_k}  {\rm KL}[\rho_{a}^{\rm GM} \Vert \rho_{\rm post}],
   \\
   \langle \nabla_a {\rm KL}[\rho_{a}^{\rm GM} \Vert \rho_{\rm post}], d\rangle\, [{\rm FIM}(a)d]_{w_k} &= 
   \langle d, {\rm FIM}(a)d\rangle\, \nabla_{w_k}  {\rm KL}[\rho_{a}^{\rm GM} \Vert \rho_{\rm post}]   -\lambda\langle d, {\rm FIM}(a)d\rangle^{\frac{3}{2}}.
    \end{split}
\end{equation}
From \cref{eq: kkt} the direction of $d$ can be determined, which leads to the following natural gradient flow:
\begin{align}
    \label{eq:FR-GM}
  \begin{bmatrix}
  \dot{m}_{k} \\
  \dot{C}_{k} \\
  \dot{w}_{k} 
  \end{bmatrix}
  =-({\rm FIM}(a))^{-1}
  \begin{bmatrix}
  w_k\int \N_k(\theta) \Bigl( \nabla_{\theta} \log\rho_a^{\rm GM}  +\nabla_{\theta} \Phi_R \Bigr)  \dd\theta
  \\
  \frac{w_k}{2}\int \N_k(\theta) \Bigl(\nabla_{\theta}\nabla_{\theta}\log \rho_a^{\rm GM}  + \nabla_{\theta}\nabla_{\theta}\Phi_R\Bigr) \dd\theta
  \\
  \int \N_k(\theta)
  \bigl(
  \log \rho_a^{\rm GM} + \Phi_R 
  \bigr) \dd\theta  + \lambda
  \end{bmatrix}.
\end{align}
Here $\lambda$ is determined by the constraint $\sum \dot{w}_{k} = 1$. Its value depends on ${\rm FIM}(a)$.

The Fisher information matrix ${\rm FIM}(a)$ for Gaussian mixtures does not have a closed-form expression, and its inversion is computationally expensive.
To improve efficiency, diagonal approximations of the Fisher information matrix have been used in the literature~\cite[Appendix C.8]{chen2024efficient}\cite{lin2019fast}, leading to the following approximation:
\begin{equation} 
\label{eq:FI-app}
{\rm FIM}(a) \approx \textrm{Diag}\left(w_1 C_1^{-1}, ..., w_K C_K^{-1}, w_1X_1, ..., w_KX_K, \frac{1}{w_1}, ..., \frac{1}{w_K}\right).
\end{equation}
Here each $X_k$ is a 4-th order tensor, and its action on any matrix $Y \in \R^{N_{\theta}\times N_{\theta}}$ is given by 
\begin{equation}
\label{eq:FI-X}
X_k Y =     \frac{1}{4}C_k^{-1} (Y + Y^T) C_k^{-1}.
\end{equation}
We write down the derivation of Fisher information matrix approximation here.  We can get \cref{eq:FI-app}
by only keeping the diagonal blocks of ${\rm FIM}(a)$ and approximating the diagonals under the
assumptions that different Gaussian components are well separated. More precisely, for the weights items,
\begin{align*}
        \int \frac{\nabla_{w_i}\rho_a^{\rm GM}  \nabla_{w_j}\rho_a^{\rm GM}}{\rho_a^{\rm GM}}\dd \theta=\int \frac{\N_i\N_j}{\rho_a^{\rm GM}}\dd \theta\approx \delta_{i,j}\int \frac{\N_i^2}{w_i\N_i}=\delta_{i,j}\frac{1}{w_i}.
\end{align*}
Here we substitute $\rho_a^{\rm GM}$ by $w_i \N _i$ during its integration with $\N_i$ . We note that we will
keep using this approximation multiple times in the following derivations.
For the means items, 
    \begin{align*}
        \int \frac{(\nabla_{m_i}\rho_a^{\rm GM})  (\nabla_{m_j}\rho_a^{\rm GM})^\top}{\rho_a^{\rm GM}} \dd \theta
        &=\int \frac{w_iw_j\N_i\N_j C_i^{-1}(\theta-m_i)(\theta-m_j)^\top C_j^{-1}}{\rho_a^{\rm GM}}\dd \theta \\
        &\approx \delta_{i,j}\int\frac{ w_i^2\N_i^2C_i^{-1}(\theta-m_i)(\theta-m_i)^\top C_i^{-1}}{w_i\N_i}\\
        &=\delta_{i,j}w_iC_i^{-1}.
    \end{align*}
For the covariances items, 
    \begin{align*}
        &\int \frac{(\nabla_{C_i}\rho_a^{\rm GM}) \otimes (\nabla_{C_j}\rho_a^{\rm GM})}{\rho_a^{\rm GM}} \dd \theta  \\
        = & \int \frac{w_iw_j\N_i\N_j \big(C_i^{-1}(\theta-m_i)(\theta-m_i)^\top C_i^{-1}-C_i^{-1}\bigr)\otimes \big(C_j^{-1}(\theta-m_j)(\theta-m_j)^\top C_j^{-1}-C_j^{-1}\bigr)}{4 \rho_a^{\rm GM}} \dd \theta  \\
        \approx &\delta_{i,j}\int w_i^2\N_i^2 \frac{\big(C_i^{-1}(\theta-m_i)(\theta-m_i)^\top C_i^{-1}-C_i^{-1}\bigr)\otimes \big(C_i^{-1}(\theta-m_i)(\theta-m_i)^\top C_i^{-1}-C_i^{-1}\bigr)}{4 w_i\N_i}  \dd \theta\\
        = & \delta_{i,j}w_i X_i,
    \end{align*}
where
    \begin{align*}
        X_i = \frac{1}{4}\int \big(C_i^{-1}(\theta-m_i)(\theta-m_i)^\top C_i^{-1}-C_i^{-1}\bigr)\otimes \big(C_i^{-1}(\theta-m_i)(\theta-m_i)^\top C_i^{-1}-C_i^{-1}\bigr) \N_i(\theta)\dd \theta
    \end{align*}
    is a 4-order tensor. Similar with the derivations of \cref{eq:Gaussian-FIM} , we have that  the act of $X_i$ on $\forall Y\in \R^{N_\theta\times N_\theta}$ is given by 
    $ X_iY = \frac{1}{4}C_i^{-1}(Y+Y^\top)C_i^{-1} $. 

Substituting the approximated Fisher information matrix \cref{eq:FI-app} into the natural gradient flow \cref{eq:FR-GM} leads to the following equations:
\begin{equation}
\begin{split} 
\label{eq:NGFlow-appendix}
        \dot{m}_{k} 
        &= -C_k\int \N_k(\theta) \Bigl( \nabla_{\theta} \log\rho_{a_t}^{\rm GM}  +  \nabla_{\theta} \Phi_R \Bigr)  \dd\theta,
        \\
        \dot{C}_{k} 
        &= -C_k\Bigl(\int \N_k(\theta)\bigl(\nabla_{\theta}\nabla_{\theta}\log \rho_{a_t}^{\rm GM}  + \nabla_{\theta}\nabla_{\theta}\Phi_R\bigr) \dd\theta\Bigr) C_k,
        \\
        \dot{w}_{k} &= -w_k\int \Bigl(\N_k(\theta) -  \rho_{a_t}^{\rm GM}\Bigr)\bigl(\log \rho_{a_t}^{\rm GM}  + \Phi_R \bigr) \dd\theta. 
\end{split}
\end{equation}
Here $\lambda = -\int \rho_a^{\rm GM}  
  \bigl(
  \log \rho_{a_t}^{\rm GM} + \Phi_R 
  \bigr) \dd\theta$.
Similar to \cref{eq:dC_t^-1}, the covariance evolution equation can be rewritten as:
\begin{equation} 
\begin{aligned}
\frac{\dd C_{k}^{-1}}{\dd t}
        &=   \int \N_k(\theta) \bigl(\nabla_{\theta}\nabla_{\theta}\log \rho_{a_t}^{\rm GM}  + \nabla_{\theta}\nabla_{\theta}\Phi_R\bigr) \dd\theta.
\end{aligned}
\end{equation}
And the weight evolution equation can be reformulated as
\begin{equation}
\begin{aligned}
\frac{\dd \log w_{k}}{\dd t} &= -\int \Bigl(\N_k(\theta) -  \rho_{a_t}^{\rm GM}\Bigr)\bigl(\log \rho_{a_t}^{\rm GM}  + \Phi_R \bigr) \dd\theta
\end{aligned}
\end{equation}
to ensure the weights remain positive.
In the following, we also refer to \cref{eq:NGFlow-appendix}  as the natural gradient flow, although involving an approximation of the Fisher information matrix.
It is worth noting that when the mode number is $K = 1$, the natural gradient flow with Gaussian mixture approximation~\cref{eq:NGFlow-appendix} reduces to the natural gradient flow with Gaussian approximation~\cref{eq:Gaussian Fisher-Rao}.

Designing effective schemes to discretize~\cref{eq:NGFlow-appendix} remains challenging, particularly due to issues such as the computation of the Hessian matrix, the collapse of different modes, and the singularity of the covariance. A stable, derivative free approximation of~\cref{eq:NGFlow-appendix} is the main focus of the present work.

\end{document}